\pdfoutput=1

\documentclass[11pt]{article}

\usepackage{EMNLP2023}

\usepackage{times}
\usepackage{latexsym}
\PassOptionsToPackage{hyphens}{url}\usepackage{hyperref}

\usepackage{rycolab}
\usepackage{partialorder}

\usepackage[T5,T1]{fontenc}

\usepackage[utf8]{inputenc}

\usepackage{microtype}

\usepackage{inconsolata}

\tikzset{cross/.style={cross out, draw, 
         minimum size=2*(#1-\pgflinewidth), 
         inner sep=0pt, outer sep=0pt,line width=1.5pt}}
\newcommand{\bluecircle}{\tikz\draw[blue,fill=blue] (0,0) circle (.5ex);}
\newcommand{\redcross}{\tikz\draw[red] (0,0) node[cross=.7ex]{};}

\newcommand{\Hexataggertt}{{\color{MacroColor}\texttt{Hexa}}}
\newcommand{\Biaffinett}{{\color{MacroColor}\texttt{Biaff}}}

\definecolor{spink}{HTML}{CC6677}
\definecolor{sred}{HTML}{F8CCCE}
\definecolor{syellow}{HTML}{DDCC77}
\definecolor{sblue}{HTML}{88CCEE}
\definecolor{sgreen}{HTML}{44AA99}
\definecolor{sorange}{HTML}{FFE6CC}

\newcommand{\bw}{{\color{MacroColor}\boldsymbol{w}}}

\newcommand{\token}{{\color{MacroColor} w}}

\newcommand{\bzero}{{\color{MacroColor} \boldsymbol{0}}}

\newcommand{\objective}{{\color{MacroColor} \calL}}

\newcommand{\param}{{\color{MacroColor}\boldsymbol{\theta}}}

\DeclareMathOperator*{\argmin}{argmin}

\newcommand{\wx}{{\color{MacroColor} x}}
\newcommand{\wy}{{\color{MacroColor} y}}
\newcommand{\wz}{{\color{MacroColor} z}}

\newcommand{\wxone}{{\color{MacroColor} \wx_1}}
\newcommand{\wxtwo}{{\color{MacroColor} \wx_2}}
\newcommand{\wyone}{{\color{MacroColor} \wy_1}}
\newcommand{\wytwo}{{\color{MacroColor} \wy_2}}

\newcommand{\tagspace}{{\color{MacroColor} \calC}}

\newcommand{\alphabet}{{\color{MacroColor} \Sigma}}

\newcommand{\kleene}[1]{{\color{MacroColor}#1^*}}
\newcommand{\str}{{\color{MacroColor} \bw}}

\newcommand{\VER}{{\color{MacroColor} \left( \vertices, \edges, \relations \right)}}
\newcommand{\VE}{{\color{MacroColor} \left( \vertices, \edges \right)}}
\newcommand{\VERel}{{\color{MacroColor} \left( \vertices, \edges, \rel \right)}}

\newcommand{\realizer}{{\color{MacroColor}{R}}}

\newcommand{\graph}{{\color{MacroColor}\calG}}

\newcommand{\vertices}{{\color{MacroColor} V}}

\newcommand{\verticesr}{{\color{MacroColor} \vertices^{r}}}
\newcommand{\verticesb}{{\color{MacroColor} \vertices^{b}}}
\newcommand{\verticeshat}{{\color{MacroColor} \widehat{\vertices}}}

\newcommand{\source}{{\color{MacroColor} \calN}}
\newcommand{\sink}{{\color{MacroColor} \calM}}

\newcommand{\edges}{{\color{MacroColor} E}}

\newcommand{\edgeshat}{{\color{MacroColor} \widehat{\edges}}}

\newcommand{\relations}{{\color{MacroColor} {R}}}

\newcommand{\poset}{{\color{MacroColor} \calP}}

\newcommand{\posetp}{{\color{MacroColor} \calP_{\textnormal{p}}}}
\newcommand{\posets}{{\color{MacroColor} \calP_{\textnormal{s}}}}
\newcommand{\toset}{{\color{MacroColor} \calT}}

\newcommand{\tosetk}{{\color{MacroColor} \calT_k}}

\newcommand{\R}{{\color{MacroColor} \mathbb{R}}}

\newcommand{\N}{{\color{MacroColor} \mathbb{N}}}
\newcommand{\Nlessthan}[1]{{\color{MacroColor} [#1]}}

\newcommand{\calC}{{\color{MacroColor}\mathcal{C}}}

\newcommand{\calF}{{\color{MacroColor}\mathcal{F}}}
\newcommand{\calG}{{\color{MacroColor}\mathcal{G}}}

\newcommand{\calL}{{\color{MacroColor}\mathcal{L}}}
\newcommand{\calM}{{\color{MacroColor}\mathcal{M}}}
\newcommand{\calN}{{\color{MacroColor}\mathcal{N}}}
\newcommand{\calP}{{\color{MacroColor}\mathcal{P}}}

\newcommand{\calS}{{\color{MacroColor}\mathcal{S}}}
\newcommand{\calT}{{\color{MacroColor}\mathcal{T}}}

\newcommand{\calX}{{\color{MacroColor}\mathcal{X}}}

\newcommand{\seqlen}{{\color{MacroColor} N}}

\newcommand{\pairFunc}{{\color{MacroColor} F_{\param}}} %

\newcommand{\rankfunc}{\color{MacroColor} f_{\param}}

\newcommand{\rankfunck}{{\color{MacroColor} \rankfunc^{(k)}}}

\newcommand{\funcrealizer}{{\color{MacroColor} \boldsymbol{\calF}_{\!\param} }}
\newcommand{\orderdim}{{\color{MacroColor} K}}

\newcommand{\orderdimfunc}{{\color{MacroColor} D}}

\newcommand{\defeq}{\mathrel{\stackrel{\textnormal{\tiny def}}{=}}}

\newcommand{\extreetrav}[0]{%
\small
\begin{tikzpicture}[baseline={(0,0)},x=0.08\linewidth, y=0.08\linewidth, minimum size=14pt][%
\draw[] (1.5, 1.6) node[shape=circle,draw,thick,inner sep=1pt] (NB) {$b$};
\draw[] (2.7, 0.3) node[shape=circle,draw,thick,inner sep=1pt] (NC) {$c$};

\draw[] (3.6, 2.7) node[shape=circle,draw,thick,inner sep=1pt] (ND) {$d$};

\draw[] (0.5, 0.3) node[shape=circle,draw,thick,inner sep=1pt] (NA) {$a$};
\draw[] (5.7, 1.6) node[shape=circle, draw,thick, inner sep=1pt] (NF) {$f$};
\draw[] (4.8, 0.3) node[shape=circle,draw,thick, inner sep=1pt] (NE) {$e$};
\draw[] (6.6, 0.3) node[shape=circle,draw,thick, inner sep=1pt] (NG) {$g$};

\draw[thick] (NA) -- (NB);
\draw[thick] (NB) -- (ND);
\draw[thick] (NC) -- (NB);
\draw[thick] (NE) -- (NF);
\draw[thick] (NF) -- (ND);
\draw[thick] (NF) -- (NG);
\end{tikzpicture}}

\newcommand{\extreetravpoone}[0]{%
\small
\begin{tikzpicture}[baseline={(0,0)},x=0.08\linewidth, y=0.08\linewidth, minimum size=14pt][%
\draw[] (0.5, 0.3) node[shape=circle,draw,thick,inner sep=1pt] (NA) {$a$};
\draw[] (1.5, 1.6) node[shape=circle,draw,thick,inner sep=1pt] (NB) {$b$};
\draw[] (2.7, 0.3) node[shape=circle,draw,thick,inner sep=1pt] (NC) {$c$};
\draw[] (3.6, 2.7) node[shape=circle,draw,thick,inner sep=1pt] (ND) {$d$};

\draw[] (5.7, 1.6) node[shape=circle, draw,thick, inner sep=1pt] (NF) {$f$};
\draw[] (4.8, 0.3) node[shape=circle,draw,thick, inner sep=1pt] (NE) {$e$};
\draw[] (6.6, 0.3) node[shape=circle,draw,thick, inner sep=1pt] (NG) {$g$};

\draw[thick,-stealth] (NA) -- (NB);
\draw[thick,-stealth] (NB) -- (ND);
\draw[thick,-stealth] (NC) -- (ND);
\draw[thick,-stealth] (NE) -- (NF);
\end{tikzpicture}}

\newcommand{\extreetravpotwo}[0]{%
\small
\begin{tikzpicture}[baseline={(0,0)},x=0.08\linewidth, y=0.08\linewidth, minimum size=14pt][%
\draw[] (0.5, 0.3) node[shape=circle,draw,thick,inner sep=1pt] (NA) {$a$};
\draw[] (1.5, 1.6) node[shape=circle,draw,thick,inner sep=1pt] (NB) {$b$};
\draw[] (2.7, 0.3) node[shape=circle,draw,thick,inner sep=1pt] (NC) {$c$};
\draw[] (3.6, 2.7) node[shape=circle,draw,thick,inner sep=1pt] (ND) {$d$};

\draw[] (5.7, 1.6) node[shape=circle, draw,thick, inner sep=1pt] (NF) {$f$};
\draw[] (4.8, 0.3) node[shape=circle,draw,thick, inner sep=1pt] (NE) {$e$};
\draw[] (6.6, 0.3) node[shape=circle,draw,thick, inner sep=1pt] (NG) {$g$};

\draw[thick,-stealth] (NG) -- (NF);
\draw[thick,-stealth] (NF) -- (ND);
\draw[thick,-stealth] (NE) -- (ND);
\draw[thick,-stealth] (NC) -- (NB);
\end{tikzpicture}}

\newcommand{\ttsort}{{\color{MacroColor} \tt{sort}}}

\newcommand{\ttkey}{{\color{MacroColor} \tt{key}}}

\newcommand{\sortedvertices}{{\color{MacroColor} U}}

\DeclareMathOperator{\semiringplus}{{\color{MacroColor} \bigoplus}}

\DeclareMathOperator{\rel}{{\color{MacroColor} \prec }}
\DeclareMathOperator{\norel}{{\color{MacroColor} \nprec }}

\newcommand{\labelrel}{{\color{MacroColor} r}}
\newcommand{\labelrelprime}{{\color{MacroColor} \labelrel' }}

\title{Linear-Time Modeling of Linguistic Structure: \\ An Order-Theoretic Perspective}

\author{Tianyu Liu \qquad Afra Amini \qquad Mrinmaya Sachan \qquad Ryan Cotterell \\
\setlength{\fboxsep}{2.5pt}%
 \setlength{\fboxrule}{2.5pt}%
 \fcolorbox{white}{white}{
     $\{$\texttt{\href{mailto:tianyu.liu@inf.ethz.ch}{tianyu.liu}, }%
     \texttt{\href{mailto:afra.amini@inf.ethz.ch}{afra.amini}, }%
     \texttt{\href{mailto:mrinmaya.sachan@inf.ethz.ch}{mrinmaya.sachan}, }%
    \texttt{\href{mailto:ryan.cotterell@inf.ethz.ch}{ryan.cotterell}}%
     $\}$\texttt{@inf.ethz.ch}
 } \\
     {%
 \setlength{\fboxsep}{2.5pt}%
 \setlength{\fboxrule}{2.5pt}%
 \fcolorbox{white}{white}{
     \includegraphics[width=.15\linewidth]{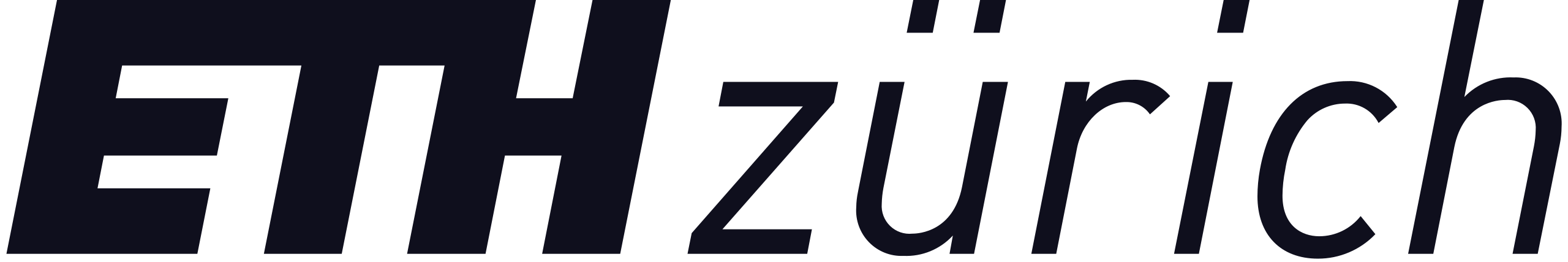}
 }
 }}

\pgfplotsset{compat=1.18}
\begin{document}

\maketitle

\begin{abstract}
Tasks that model the relation between pairs of tokens in a string are a vital part of understanding natural language.
Such tasks, in general, require exhaustive pair-wise comparisons of tokens, thus having a quadratic runtime complexity in the length of the string.
We show that these exhaustive comparisons can be avoided, and, moreover, the complexity of such tasks can be reduced to linear by casting the relation between tokens as a partial order over the string.
Our method predicts real numbers for each token in a string in parallel and sorts the tokens accordingly, resulting in total orders of the tokens in the string.
Each total order implies a set of arcs oriented from smaller to greater tokens, sorted by their predicted numbers.
The intersection of total orders results in a partial order over the set of tokens in the string, which is then decoded into a directed graph representing the desired linguistic structure.
Our experiments on dependency parsing and coreference resolution show that our method achieves state-of-the-art or comparable performance.
Moreover, the linear complexity and parallelism of our method double the speed of graph-based coreference resolution models, and bring a 10-times speed-up over graph-based dependency parsers.
\newline
 \newline
 \vspace{0.2em}
  \hspace{.5em}\includegraphics[width=1.05em,height=1.05em]{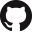}\hspace{.5em}\parbox{\dimexpr\linewidth-2\fboxsep-2\fboxrule}
  {\small \url{https://github.com/lyutyuh/partial}}
\end{abstract}

\section{Introduction}
\def\yfirstline{2}
\def\ysecondline{-1}
\begin{figure*}[t]
\centering
\begin{subfigure}[b]{0.23\textwidth}
    \includegraphics[trim={12mm 3mm 12mm 0},clip,width=\textwidth]{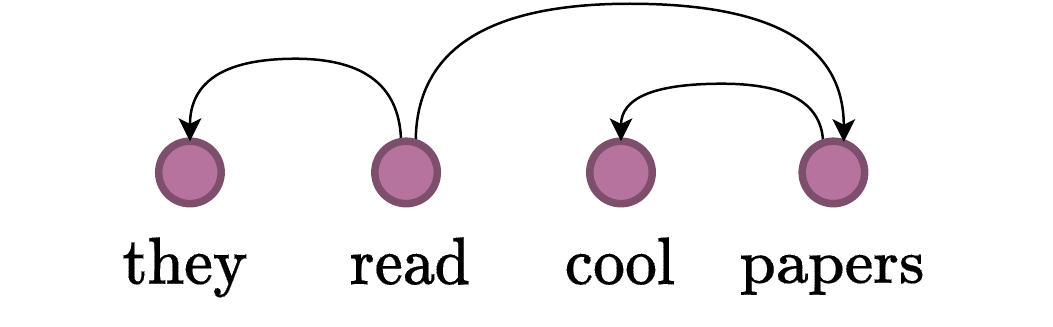}
    \caption{An example dependency structure. The root arc and arc labels are omitted.} \label{fig:illu-dep-structure}
\end{subfigure}\hspace{2mm}
\begin{subfigure}[b]{0.20\textwidth}
     \includegraphics[trim={10mm 11.5mm 10mm 2mm},clip,width=\textwidth]{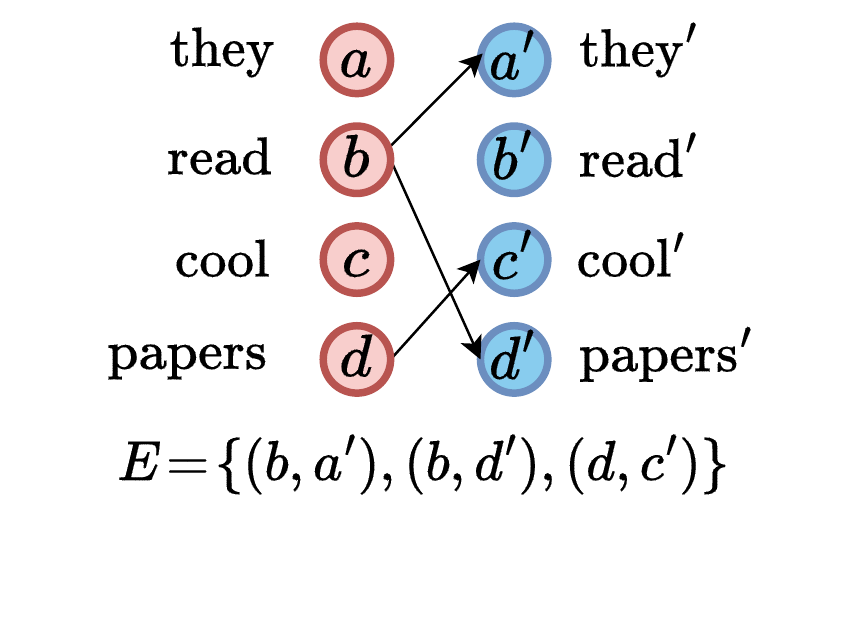}
    \caption{The token-split structure of \cref{fig:illu-dep-structure}, which is a partially ordered set.} \label{fig:illu-token-split}
\end{subfigure}\hspace{2mm}
\begin{subfigure}[b]{0.48\textwidth}
     \includegraphics[trim={15mm 8mm 10mm 3mm},clip, width=\textwidth]{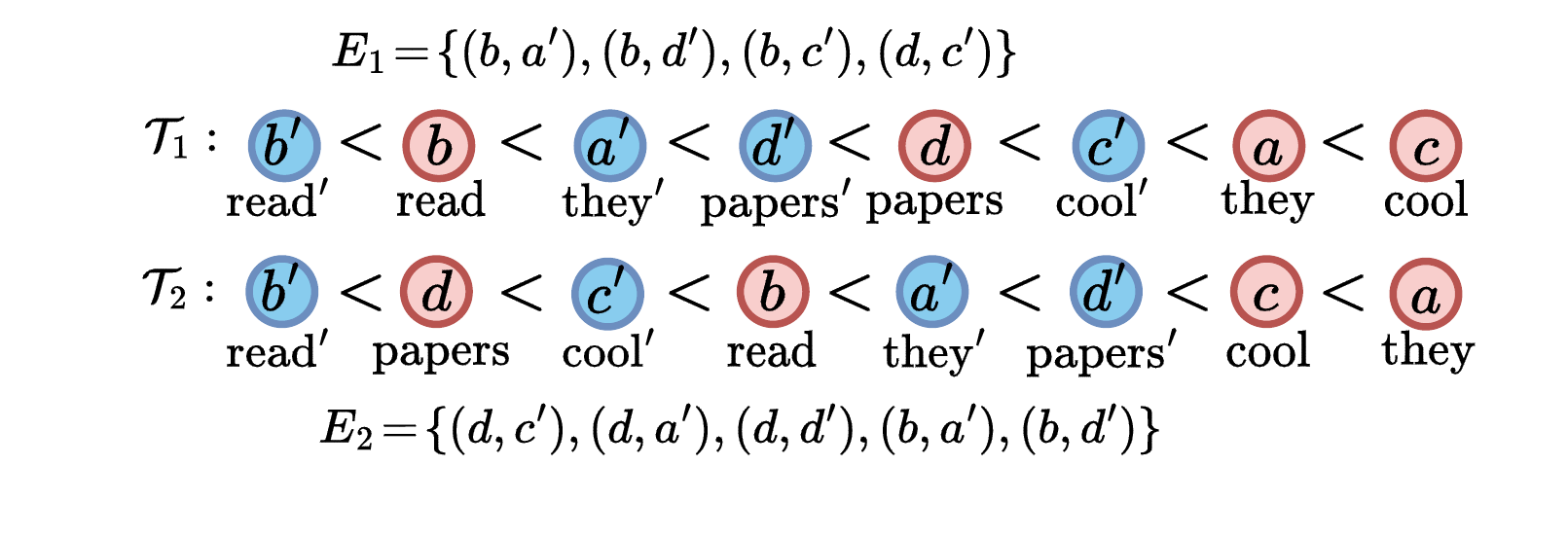}
    \caption{A realizer of \cref{fig:illu-token-split} with $2$ total orders such that $\edges=\edges_1 \cap \edges_2$. $\edges_1$ and $\edges_2$ contain the arcs oriented from $\verticesr$ (\colorbox{sred}{\makebox(7,3){\rule[-3pt]{0pt}{0pt}red}} nodes) to $\verticesb$ (\colorbox{sblue}{\makebox(10,3){\rule[-3pt]{0pt}{0pt}blue}} nodes) and from left to right. } \label{fig:illu-realizer}
\end{subfigure}
\caption{An overview of our method. To model a linguistic structure, represented as a directed graph in \cref{fig:illu-dep-structure}, we first convert it into a token-split structure (see \cref{sec:tss}) in \cref{fig:illu-token-split}, which is a partial order, to remove undesired transitivity. Then, $2$ real numbers are predicted for each vertex in \cref{fig:illu-token-split}. The positions of vertices in \cref{fig:illu-realizer} in the inequalities indicate the real numbers the vertices are mapped to. The vertices are sorted twice accordingly, resulting in a realizer (see \cref{def:realizer})  of $2$ total orderings, each possessing a set of edges $\edges_1$ and $\edges_2$. The exact set of desired edges in the original structure can be restored from the intersection of $\edges_1$ and $\edges_2$ (see \cref{sec:efficiency-guarantees}). Some qualitative examples are included in \cref{app:qual-ex}. \looseness=-1}
\label{fig:overview}
\end{figure*}

Strings of tokens in natural language are not constructed arbitrarily.
Indeed, which tokens co-occur within the same string is highly structured according to the rules of the language. Understanding such structures is critical to the comprehension of natural language. 
In natural language processing (NLP), many structured prediction tasks aim to automatically extract the underlying structure that dictates the relationship between the tokens in a string of text. Examples of such tasks include dependency parsing, semantic parsing, and coreference resolution. These tasks involve predicting complex and hierarchical output structures, making them inherently more challenging than their classification or regression counterparts. This paper contributes a novel and generic framework for structured prediction with empirical evidence from dependency parsing and coreference resolution. \looseness=-1

Many machine learning models for structured prediction score and predict graphs \citep{mcdonald-etal-2005-non,mcdonald-pereira-2006-online}, in which the vertices represent the tokens in the string and the edges represent the relations between them. 
One common strategy to model a graph is to decompose it into smaller subgraphs that are tractable \cite[\S 2.2]{taskar-etal-2004-max,smith-linguistic-structure}. 
For example, arc-factored models \cite{eisner-1996-three} score a graph only using the score of each constituent edge. 
However, even with such simplification, the computational costs of arc-factored models are superlinear.
The reason is that one needs to exhaustively compute scores for all possible edges in the graph, which, in general, requires at least \emph{quadratic} number of computations with respect to the length of the string. 
Another common strategy employs weighted transition-based systems \cite{knuth-shift-reduce,yamada-matsumoto-2003-statistical,nivre-2003-efficient}.
They decompose structures into transitions between intermediate model states and \emph{do} offer linear-time algorithms. 
However, in general, predicting the transitions between states cannot be parallelized, which is another worrying limitation. 
The authors of this paper contend the limitations of both graph-based and transition-based models are frustrating in an era when researchers are processing longer and longer texts \cite{long-range-arena}.\looseness=-1

From a more abstract perspective, the mathematical and algorithmic foundation on which structured prediction models rest can be regarded as a \emph{design choice}.
Graph-based and transition-based modeling are both specific design choices.
These design choices impose substantial \emph{inductive biases} by confining the class of models available to be utilized to solve the task and set limits on the efficiency of the models.
In this paper, we propose a fresh design choice for structured prediction.
Specifically, we propose an \emph{order-theoretic} perspective to understand and model structures in NLP.
Our approach can predict many structures in natural language in $\bigo{\seqlen}$ time where $\seqlen$ is the length of the string and is easily parallelizable.
The linear-time complexity means our method avoids comparing all $\bigo{\seqlen^2}$ token pairs. 
The key innovation that enables this speed-up is the following: 
Rather than considering structures as graphs, we view them as \emph{partial orderings} of the tokens in the strings.\looseness=-1

Concretely, we treat structured prediction as a regression task. 
Because the set of real numbers $\R$ is naturally ordered by $<$, we use real numbers as the proxy for determining the partial order. 
We predict $\orderdim$ numbers for each token and sort the tokens $\orderdim$ times accordingly. 
Two tokens are partially ordered by $\rel$ if and only if they are ordered by $<$ in \emph{all} of the $\orderdim$ orders above. We further provide an efficiency guarantee based on the well-established result in order theory that partial orders satisfying particular conditions can be represented as the intersection of as few as $\orderdim=2$ total orders. We show that most structures in natural language, including trees, alignments, and set partitions, satisfy these conditions. This result enables us to develop a linear-time algorithm for predicting such structures. \cref{fig:overview} gives an illustrative example of our framework applied to dependency parsing, in which the structure being modeled is a tree.\looseness=-1

On dependency parsing, our experimental results show that our method achieves $96.1$ labeled attachment score (LAS) and $97.1$ unlabeled attachment score (UAS) by using an intersection of only $2$ total orders, $96.4$ LAS and $97.4$ UAS using an intersection of $4$ total orders on the English Penn Treebank \cite{marcus-etal-1993-building}.
Furthermore, our method sets the new state of the art on Universal Dependencies 2.2 \cite{ud2.2}, while being 10 times faster and more memory efficient than graph-based models. Our method also achieves $79.2$ F1 score with only $4$ total orders on the English OntoNotes coreference resolution benchmark \cite{pradhan-etal-2012-conll}, which is on par with the state of the art, while being twice as fast and using less memory.

\section{Motivation}
We now provide high-level motivation for order-theoretic structured prediction.\looseness=-1

\subsection{Linearization of Structure} \label{sec:motivation-linearization}
The NLP literature abounds with linear-time structured prediction models.
Many are derived from the classical shift--reduce parsers \citep{knuth-shift-reduce} from the compiler literature.
One recent line of research has derived linear-time parsers by reducing parsing to tagging \cite[][\emph{inter alia}]{gomez-rodriguez-vilares-2018-constituent,strzyz-etal-2020-bracketing,kitaev-klein-2020-tetra,hexa}.
In these methods, a \emph{finite} set of tags $\tagspace$ is chosen such that all structures for parsing a string can be embedded in $\tagspace^\seqlen$ for a string of length $\seqlen$.
Tagging-based parsers often yield strong empirical performance in both constituency parsing and projective dependency parsing.
A natural question is, then, why do we need another method? \looseness=-1

We give two motivations. 
The first linguistic and the second mathematical.
Linguistically, the underlying structures of natural language, e.g., syntax, semantics, and discourse, are often \emph{not} aligned with the surface form of a sequence due to the existence of \defn{displacement} \cite[][Chapter 1, p. 44]{minimalist}.
The strong performance of parsing-as-tagging schemes relies, in part, on there being a tight correspondence between the surface string and structure \citep[][Proposition 1]{afra-parsing}.
Mathematically, the maximum number of structures that a discrete tag sequence can represent is at most $\bigo{|\tagspace|^\seqlen}$.
This set is simply not large enough to capture many structures of interest in NLP.
For instance, the space of non-projective dependency trees of $\seqlen$ tokens has a cardinality of $\bigo{\seqlen^{\seqlen-2}}$ \cite{cayley-trees}. 
Therefore, to parse non-projective dependency trees with tagging, the size of the tag set has to grow with $N$. 
However, this implies performing a classification task with an \emph{infinite} number of classes.\looseness=-1

\subsection{An Illuminating Example}\label{sec:illuminating-binary}
Order-theoretic approaches appear across computer science.
For instance, it is well-known that a binary tree can be uniquely restored from its inorder traversal and either the pre- or postorder traversal.
Consider the following binary tree.\looseness=-1
\begin{myexample}[Binary Tree] \label{eg:binary-tree}
\vphantom \\

\captionsetup{type=figure}
\noindent
\resizebox{.5\textwidth}{!}{%
    \begin{tabular}[t]{@{\hskip1pt}p{0.25\textwidth}@{\hskip1pt}p{0.25\textwidth}@{\hskip1pt}}
    \multirow{1}{0.36\textwidth}{\extreetrav} & 
    \multirow{1}{0.36\textwidth}{\extreetravpoone} \\[35pt]
    \multicolumn{2}{c}{{\small Inorder : $abcdefg$} \hspace{10pt} {\small Postorder: $acbegfd$}}
    \end{tabular}
}

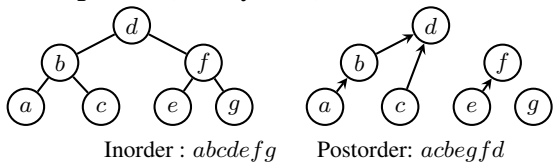
\captionof{figure}{An example binary tree and a partial order over the vertices induced by two total orders.}     
\vspace{-5pt}
\end{myexample} 
In a binary tree, a vertex $\wx$ is a left descendant of vertex $\wy$ \emph{if and only if} $\wx$ is visited before $\wy$ in both of the in- and postorder traversal.
E.g., in \Cref{eg:binary-tree}, $a$ is the left descendant of $d$ and is visited before $d$ in both the in- and postorder traversal.

Another way of stating the above fact is that a binary tree can be recovered from the \emph{combination of two total orders}, the one induced by the inorder traversal and the one induced by the postorder traversal.
Combining these two total orders yields a partial order, i.e., left descendant, from which the left child of each vertex can be identified.
This partial order is shown on the right of \cref{eg:binary-tree}.
See \cref{app:revisit-binary-traversal} and \citep[\S2.3.1, Ex.~7]{knuth-comp-programming} for further discussion. 
In light of these observations, we conceive an order-theoretic treatment that constructs a tree by predicting multiple total orders and intersecting them.
In terms of computation, predicting total orders only requires labeling each node with real numbers and then sorting, the complexity of which is linear under radix sort.
On the other hand, an arc-factored model necessarily computes all $\bigo{\seqlen^2}$ pair-wise scores for every pair of vertices to decide the existence of each edge.\looseness=-1

Next, we generalize the intuitions gained from this example. In \Cref{sec:order-theory}, we explore the class of graphs that can be efficiently represented with partial orders. In \Cref{sec:parameterization}, we show how to learn the ordering efficiently with neural networks.\looseness=-1

\section{Order and Structure}
\label{sec:order-theory}

In this section, we describe an order-theoretic treatment for linguistic structure prediction.
Specifically, we treat the structure to be predicted as a partially ordered set, i.e., a set equipped with a \defn{transitive} relation $\rel$.
We begin by revisiting how linguistic structures are represented as graphs.

\subsection{Linguistic Structures as Directed Graphs}\label{sec:stru-digraph}
Let $\alphabet$ be an alphabet, i.e., a finite set of natural language tokens, and let $\str = \token_1 \token_2 \cdots \token_{\seqlen} \in \kleene\alphabet$ be a string.
Linguistic structure prediction is the task of assigning a structure, e.g., a dependency tree, to a given string $\str$ in natural language.

A wide range of linguistic structures are built upon the relations between pairs of tokens.
Many structured prediction models are thus arc-factored, i.e., they predict the arcs between a pair of tokens and then combine them back into structures, which are our focus in this work.
Formally, their major goal is to model the homogeneous relation\footnote{A \defn{homogeneous relation} on a set $\calX$ is a binary relation between two elements in $\calX$. It can be equivalently represented with the set of edges in a graph in which $\calX$ is the set of vertices.\looseness=-1} on the \defn{spanning node set} $ \vertices = \{\token_1, \token_2, \cdots, \token_{\seqlen}\}$ of a sentence $\str = \token_1 \cdots \token_{\seqlen}$ \cite{dep-parsing}.
The output space is defined by the input itself, in contrast to the external label spaces in other tasks such as classification or language generation.\looseness=-1

\begin{definition}[Structure]\label{def:structure}
    A \defn{structure} over a string $\str = \token_1 \token_2 \cdots \token_{\seqlen}$ is a directed graph $\graph = \VE$,
    where $\vertices = \{ \token_1, \token_2, \cdots, \token_{\seqlen} \}$,
    $\edges \subseteq \vertices \times \vertices $ is the set of arcs.
    A \defn{typed structure} $\graph = \VER$ is a structure with $\edges \subseteq \vertices \times \vertices \times \relations$, where $\relations$ is a finite set of relation labels.
\end{definition}

Most linguistic structures are naturally subsumed under this definition. 
We give two examples of linguistic structure prediction tasks. \looseness=-1

\begin{myexample}[{\citealp[Dependency Parsing;][Def. 2.3]{dep-parsing}}]\label{ex:dep-parsing}
    A \defn{dependency structure} is a structure $\graph = \VER$, where $\edges \subseteq \vertices \times \vertices \times \relations $, and $\relations$ is the set of dependency relation types.
    If $(\wx, \wy, \labelrel ) \in \edges $, then $ \forall \labelrel' \neq \labelrel, (\wx, \wy, \labelrel') \notin \edges$.
\end{myexample}

\begin{myexample}[Coreference Resolution]\label{ex:coreference}
    A \defn{coreference structure} is a structure $\graph = \VER$, where $\edges \subseteq \vertices \times \vertices \times \relations $, and $\relations = \{\labelrel, \labelrelprime \}$. 
    The relations $\labelrel, \labelrelprime$ represent the entity mention and coreference, respectively.
    We have $(\wx, \wy, \labelrel) \in \edges$ if and only if the textual span $\wx\!:\!\wy$ in $\str$ is a mention of an entity.
    $(\wxone, \wxtwo, \labelrelprime) \!\in\! \edges \land (\wyone, \wytwo, \labelrelprime) \!\in\! \edges $ if and only if the textual spans  $\wxone\!:\!\wyone$ and $\wxtwo\!:\!\wytwo$ corefer. 
\end{myexample}

\subsection{From Directed Graphs to Partial Orders}\label{sec:dg-to-poset}

Our treatment constructs linguistic structures with techniques from order theory. 
The key is to cast the relation between tokens as an order, which is defined as follows. \looseness=-1
\begin{definition}[{\citealp[Order;][]{hausdorff1914grundzüge}}]\label{def:order}
An \defn{order} over a set $\vertices$ is a relation $\rel$ such that the following hold for all $\wx, \wy, \wz \in \vertices$:\looseness=-1
    \begin{enumerate}[noitemsep,nosep,label=\textnormal{(\alph*)},wide]
        \item \defn{irreflexivity}: $\wx \norel \wx$;
        \item \defn{asymmetry}: $\wx \rel \wy \implies \wy \norel \wx$;
        \item \defn{transitivity}: $\wx \rel \wy \land \wy \rel \wz \implies \wx \rel \wz$.
    \end{enumerate}
\end{definition}

Natural language exhibits structural sparsity in that each token in a string usually only interacts with very few other tokens with a particular relation.
For instance, in a dependency graph, there are no direct paths between most of the word pairs.
Such sparsity, from an order-theoretic point of view, can be characterized by \defn{incomparability} in a partially ordered set \cite[Chapter~1, p.~2]{birkhoff-1967-lattice}.

By analogy, we define the following \defn{partially ordered structure}, which is a partially ordered set mathematically.
Its elements are the tokens of a string, and its order encodes a linguistic structure.

\begin{definition}[Partially Ordered Structure] \label{def:partial-order}
    Let $ \graph = (\vertices, \edges) $ be a structure.
    Define the following relation $\rel$: For $\wx, \wy \in \vertices$, $\wx \rel \wy \iff (\wx, \wy) \in \edges$.
    We call $\poset = \VERel$ a \defn{partially ordered structure } if $\rel$ satisfies \cref{def:order}.\looseness=-1
\end{definition}

The essential theoretical foundation of our linguistic structure prediction framework is the classic result that partial orders can be represented by an intersection of total orders \cite{poset-dushnik-miller}. %
It is this result that enables us to use real numbers as a proxy to determine the partial ordering of tokens. \looseness=-1

\begin{definition}[Totally Ordered Structure]\label{def:total-order}
A partially ordered structure $\poset = \VERel$ is \defn{totally ordered} if $\forall \wx, \wy \in \vertices \colon \wx \rel \wy \lor  \wy \rel \wx$.\looseness=-1
\end{definition}
Due to the transitivity of the ordering relation $\rel$, a totally ordered structure of $|\vertices|$ elements always contains $|\edges| = \binom{|\vertices|}{2} \vspace{2pt}$ relations. 
Given a collection of structures $\{(\vertices, \edges_k)\}_{k \in \Nlessthan{\orderdim}}$
defined over the same set of vertices $V$, their \defn{intersection} is also a structure---namely $(\vertices, \cap_{k \in \Nlessthan{\orderdim}} \edges_k)$, where $ \orderdim \in \N, \Nlessthan{\orderdim} \defeq \{1,\cdots,\orderdim\}$. 
The intersection of partially ordered structures remains partially ordered.\looseness=-1

We now cite a famous theorem from order theory.
\begin{theorem}[\citet{szpilrajn}] \label{thm:szpilrajn}  
    \hspace{5pt}
    Every partially ordered structure is contained in a totally ordered structure, i.e., for every partially ordered structure $\poset = \VERel$, there exists a totally ordered structure $\toset = ( \vertices, \edgeshat, \rel )$ such that $\edges \subseteq \edgeshat$.\looseness=-1
\end{theorem}

\Cref{thm:szpilrajn} ensures that every partially ordered structure can be embedded in some totally ordered structure in the sense that the totally ordered structure \emph{contains} all the relations in the partially ordered structure.
More importantly, a stronger result can be shown: Partially ordered structures can \emph{always} be represented as intersections of a collection of totally ordered structures.
\begin{definition}[Realizer] \label{def:realizer}
   Let $\poset = \VERel$ be a partially ordered structure.
    A \defn{realizer} $\realizer_{\poset}$ of $\poset$ is a set of totally ordered structures $\Big\{\toset_1, \toset_2, \cdots, \toset_{\orderdim} \Big\}$ over $V$, i.e., each $\tosetk = \left( \vertices, \edges_k, \rel_k \right)$, 
    such that $\edges = \bigcap_{k \in \Nlessthan{\orderdim}} \edges_k$.
    In other words, $\forall \wx, \wy \in \vertices, \wx \rel \wy \iff \bigwedge_{k \in \Nlessthan{\orderdim}} \wx \rel_k \wy $.
\end{definition}

\begin{theorem}[{\citealp[][Thm. 2.32]{poset-dushnik-miller}}]
    There exists a realizer $\realizer_{\poset}$ for every partially ordered structure $\poset = \VERel$.\looseness=-1
\end{theorem}
A corollary of the above theorem is that the complexity of a partially ordered structure can be characterized by its order dimension, which is defined as follows.\looseness=-1
\begin{definition}[{\citealp[Order Dimension;][]{poset-dushnik-miller}}]
Let $\poset = \VERel$ be a partially ordered structure. 
    The \defn{order dimension} $\orderdimfunc_{\poset}$ of $\poset$ is the cardinality of the smallest realizer of $\poset$.\looseness=-1
\end{definition}

\subsection{Efficiency Guarantees} \label{sec:efficiency-guarantees}
In this section, we give an efficiency guarantee of order-theoretic structured prediction.
These efficiency guarantees come from a series of results in order theory and lattice theory \cite[\emph{inter alia}]{poset-dushnik-miller,hiraguchi-55,birkhoff-1967-lattice}.

First, it is important to note that \emph{not} all partially ordered structures can be represented as an intersection of a \emph{constant} number of totally ordered structures \cite[Thm. 4.1]{poset-dushnik-miller}.

In fact, testing whether the order dimension of a partial order $\poset$ is at most $\orderdim$, $\forall \orderdim \ge 3$ is NP-complete \cite{yannakakis}.
However, we contend that most of the linguistic structures found in natural language processing \cite{smith-linguistic-structure}---including trees, equivalence classes (i.e., set partitioning), and alignment (i.e., bipartite matching)---can be represented as the intersection of $2$ totally ordered structures.
We postulate that this is possible due to their innate sparsity, i.e., a token tends to only interact with a few other tokens. 
These assumptions are formalized as follows.
\begin{assumption}[Sparsity] \label{ass:sparsity}
   A class of linguistic structures $\graph = \VE$ over natural language strings $\str \in \alphabet^*$ with $N = |\str|$ is called \defn{sparse} if $\bigo{|\edges|} = \bigo{\seqlen}$.
\end{assumption}
\begin{assumption}[Linguistic Structures are 2-dimensional] \label{ass:linguistic-2-dim}
  Structures in natural language can be represented as intersections of $2$ totally ordered structures.\looseness=-1
\end{assumption}

We justify \crefrange{ass:sparsity}{ass:linguistic-2-dim} in \cref{app:efficiency-guarantees}.
Empirical evidence is also provided in \cref{sec:experiments}, where $2$-dimensional order-theoretic models are trained to tackle two linguistic structure prediction tasks with high performance.\looseness=-1

\subsection{Token-Split Structures}\label{sec:tss}
An obvious limitation of our formulation of linguistic structures as partial orders is that 
by \Cref{def:order}, partial order is transitive. In other words, $\wx \rel \wy \land \wy \rel \wz $ implies $\wx \rel \wz$, which, however, does \emph{not} hold in the structures characterized by the directed graph formalization in \cref{def:structure}. 
In addition, we note that our notation of structures generalizes to cyclic graphs.
However, partially ordered structures are inherently acyclic due to the transitivity of $\rel$.
We now introduce the \defn{token-split structure}, which enables cycles and removes redundant edges introduced by transitivity in partially ordered structures.\looseness=-1

\begin{definition}[Token-Split Structure]\label{def:token-split}
    A \defn{token-split structure} induced by a structure $\graph = (\vertices, \edges)$ is a structure $\poset = ( \verticeshat, \edgeshat, \rel )$ such that
    \begin{deflist}[noitemsep,nosep,label=\textnormal{(\alph*)},ref=\textnormal{\thetheorem\,(\alph*)}]
        \item $\verticeshat \defeq \verticesr \cup \verticesb$, where $\verticesr = \{\wx^r \mid \wx \in \vertices \}, \verticesb = \{ \wx^b \mid \wx \in \vertices\} $; 
        \item $\verticesr \cap \verticesb = \varnothing$; \label{def:token-split-2}
        \item $\begin{aligned}[t]\edgeshat =& \left\{ ( \wx^{r}, \wy^{b} ) \mid (\wx, \wy) \in \edges \right\} .\end{aligned}$ \label{def:token-split-4}
    \end{deflist}
\end{definition}

In other words, a token-split structure maps the edges from the original structure, \emph{including self-loops}, into a bipartite graph in which the edges are oriented from $\verticesr$ to $\verticesb$. 
An example is displayed in \cref{fig:illu-token-split}.\looseness=-1

Given a token-split structure $\poset = ( \verticeshat, \edgeshat, \rel)$, we can recover the original structure $\graph = (\vertices, \edges)$ from which $\poset$ is induced using the following equation
\begin{align} \label{eq:convert-back}
\!    \edges \!\!=\!\!  \{ ( \wx, \wy) \mid \wx^{r} \in \verticesr \! \land \wy^{b} \in \verticesb \! \land \wx^{r} \rel \wy^{b} \}
\end{align}

\begin{restatable}{theorem}{thmvsispo} \label{thm:vs-is-po}
    Token-split structures are partially ordered.
\end{restatable}
\begin{proof}
    See \cref{sec:vs-is-po-proof}.
\end{proof}
\begin{remark}[Conversion between Structures and Partially Ordered Structures]
\Cref{thm:vs-is-po} and \cref{eq:convert-back} ensure that we can convert back and forth between any structure under \cref{def:structure} and a partially ordered structure.
Specifically, they enable us to first convert a structure to a partially ordered structure, predict it order-theoretically, and then finally convert it back to a structure.

\end{remark}

\section{A Neural Parameterization}
\label{sec:parameterization}
In this section, we describe the core technical contribution of our work.
We show how to model partially ordered structures with a neural model.
Specifically, we define a parameterized realizer of \cref{def:realizer} and an objective function for training the realizer to model the token-split structures.
We also give algorithms for efficient training and decoding.\looseness=-1

\subsection{Neuralized Total Order}

We now discuss a parameterized neural network that induces partial orders as the intersection of several total orders.\looseness=-1

\begin{definition}[Functional Realizer]\label{def:func-real}
    A \defn{functional realizer} of a partially ordered structure $\poset = \VERel$ is a set of mappings $\funcrealizer = \{ \rankfunc^{(1)}, \cdots, \rankfunc^{(\orderdim)} \}$, where $\param$ is the set of learnable parameters shared among $\rankfunc^{(k)}$, 
    and the order dimension $\orderdim \in \N$ is a hyperparameter of the realizer.
    The realize element $\rankfunck \colon \vertices \to \R, \; \forall k \in \Nlessthan{\orderdim}$ maps each vertex in the input structure to a real number.
    We overload $\funcrealizer$ as a mapping $\funcrealizer \colon \vertices \to \R^\orderdim$, defined as $\funcrealizer(\wx) \defeq \begin{bmatrix}\rankfunc^{(1)}(\wx), \cdots, \rankfunc^{(\orderdim)}(\wx) \end{bmatrix}^{\top}$. \looseness=-1
\end{definition}
The set of real numbers $\R$ is totally ordered, in which the order is given by the $<$ (less than) relation.
Each individual $\rankfunc^{(k)} \in \funcrealizer$ induces a total order $\tosetk = \left( \vertices, \{ (\wx,\wy) \mid \wx,\wy \in \vertices, \rankfunck(\wx) < \rankfunck(\wy) \}, \rel_k \right)$.\footnote{In this work, we assume $\rankfunck$ is injective, i.e., $\forall \wx,\wy \in \vertices, \rankfunck(\wx) \neq \rankfunck(\wy)$. See \cref{sec:lim-float} for further discussion on the practicality of this assumption.} 

The functional realizer assigns $\orderdim$ total orders $\{\toset_1, \toset_2, \cdots, \toset_{\orderdim} \}$ to the input string.
During decoding, an edge from $\wx $ to $ \wy$ exists in $\poset$ if and only if $\wx \rel_k \wy$ holds in $\tosetk, \forall k \in \Nlessthan{\orderdim}$.

Implementing \cref{def:func-real} with neural networks is straightforward.
To obtain $\rankfunck(\wx^{r})$ and $\rankfunck(\wx^{b})$, where $\wx^{r}, \wx^{b}$ are two vertices introduced by the token-split formulation (\cref{def:token-split}) corresponding to the same token $\bw_{\wx}$ in the input, we apply two linear projections on the contextualized representation of $\wx$ given by a pretrained model parameterized by $\param$.\footnote{If $\bw_{\wx}$ consists of more than one subword due to tokenization, we apply the projection to the representation of the last subword.}
In total, $2\orderdim$ real numbers are predicted for each input token. \looseness=-1

\subsection{Learning a Functional Realizer}
To learn the functional realizers with a gradient-based procedure, we need a differentiable objective.
In a partially ordered structure $\poset$ with functional realizer $\funcrealizer = \{ \rankfunc^{(1)}, \rankfunc^{(2)}, \cdots, \rankfunc^{(\orderdim)} \}$, we have $\wx \rel \wy$ if and only if $\bigwedge_{k \in \Nlessthan{\orderdim}} \left(\rankfunck(\wx) < \rankfunck(\wy)\right)$. %
We re-express this condition as follows:
\begin{equation}
\pairFunc(\wx, \wy) 
 \defeq \max_{k \in \Nlessthan{\orderdim}} \left(\rankfunck(\wx) - \rankfunck(\wy) \right) < 0   \label{eq:pairwise-func}
\end{equation}
We call $\pairFunc$ a \defn{pair-wise function}. 
On the other hand, we have $\wx \norel \wy$ if and only if $\bigvee_{k \in \Nlessthan{\orderdim}} \left(\rankfunck(\wx) \ge \rankfunck(\wy)\right)$.
This condition can be re-expressed as $ \pairFunc(\wx, \wy)  \ge 0$.
Thus, empirically, the smaller $\pairFunc(\wx, \wy)$ is, the more likely the relation $\wx \rel \wy$ exists.\looseness=-1 

We now define a training objective, which encourages the model to make decisions that comply with the order constraints enforced by the structures, described by \cref{eq:pairwise-func}.
Given the token-split version $\poset = \left( \vertices, \edges, \rel \right)$ induced by the structure being modeled, we consider the following objective
\begin{equation} \label{eq:objective}
\begin{split}
    \objective(\param) =
    \log  \sum_{(\wx, \wy) \in \vertices^2 \setminus \edges} & \exp -\pairFunc(\wx, \wy) + \\
       \log & \sum_{(\wx, \wy) \in E} \exp \pairFunc(\wx, \wy) 
\end{split}
\end{equation}
The first term maximizes $\pairFunc(\wx, \wy) $ for $\wx \norel \wy$, while the second minimizes $\pairFunc(\wx, \wy) $ for $\wx \rel \wy$.
Note that in the second term, we assume $\bigo{|\edges|} = \bigo{\seqlen}$ in a linguistic structure following \cref{ass:sparsity}.\looseness=-1

\subsection{An Efficient Algorithm} \label{sec:efficient-algorithm}

We remark that both training and decoding of the proposed model can be regarded as performing an aggregation for every token $\wx \in \vertices$.

\begin{definition}[Aggregation]
    An $\oplus$-\defn{aggregation} given a token $\wx$ for a pair-wise function $\pairFunc$ over the set $\vertices$ is an operation $\bigoplus_{\wy \in \vertices} \pairFunc(\wx, \wy)$, where $\oplus$ is a commutative and associative operation over which real number addition $+$ is distributive.
\end{definition}
Aggregation is a common abstraction for computing the relation between a token $\wx$ and every other token. 
The aggregation operator is associative and commutative, thus can be computed in parallel. 
The number of required computations is $\bigo{|\vertices|}$ for naïvely computing an aggregation of $\wx$.

During training, we $\oplus$-aggregate using negative \texttt{log-sum-exp}, i.e., we compute $-\log \sum_{\wy} \exp (-\pairFunc(\wx, \wy))$ for all $\wx$, to compute the first term of \cref{eq:objective}.
In greedy decoding, we $\oplus$-aggregate by computing $\min_{\wy} \pairFunc(\wx, \wy)$ to find the optimal relation arc from each $\wx$.
Na{\"i}vely, $\oplus$-aggregating for every token $\wx \in \vertices$ takes $\bigo{\seqlen^2}$ in total, as each aggregand has a complexity of $\bigo{\seqlen}$.
However, the partial order we assigned over $\vertices$ allows us to efficiently compute the aggregands.  \looseness=-1

For $\orderdim = 2$, we can inspect \cref{eq:pairwise-func} to see that $\pairFunc(x, y)$ is equal to \emph{either} $\rankfunc^{(1)}(\wx) - \rankfunc^{(1)}(\wy)$ or $\rankfunc^{(2)}(\wx) - \rankfunc^{(2)}(\wy)$.
We now define the following two subsets of $\vertices$ for $k \in \{ 1, 2\}$:
\begin{subequations}
\begin{align}
   \! \!\!\!\calS_{k}(\wx) &\!= \Big\{ \wy \mid  \pairFunc(\wx, \wy) = \rankfunc^{(k)}(\wx) -\rankfunc^{(k)}(\wy) \Big\}  \nonumber
\end{align}
\end{subequations}

Using this notation, we can write the following
\begin{subequations}
\begin{align}
   \!\!\! \bigoplus_{(x, y) \in \vertices^2}& \pairFunc(x, y) = 
     \bigoplus_{x \in \vertices} \bigoplus_{y \in \vertices}\pairFunc(x, y) \\
     &= \underbrace{\bigoplus_{x \in \vertices} \bigoplus_{y \in\calS_{1}(\wx)} \left(\rankfunc^{(1)}(\wx) -\rankfunc^{(1)}(\wy) \right)}_{\defeq G_1} \\
     &\quad\quad \oplus   \underbrace{\bigoplus_{x \in \vertices} \bigoplus_{y \in\calS_{2}(\wx)} \left(\rankfunc^{(2)}(\wx) -\rankfunc^{(2)}(\wy) \right)}_{\defeq G_2} \nonumber 
\end{align}
\end{subequations}
We now give an efficient algorithm to compute $G_1$ and, by symmetry, $G_2$.
Our first observation is that, by distributivity, we can write
\begin{subequations}
\begin{align}
    G_1 &= \bigoplus_{x \in \vertices} \bigoplus_{y \in\calS_{1}(\wx)} \left(\rankfunc^{(1)}(\wx) -\rankfunc^{(1)}(\wy) \right) \\
    &= \bigoplus_{x \in \vertices} \underbrace{\bigg( \rankfunc^{(1)}(\wx) + \bigoplus_{y \in\calS_{1}(\wx)} -\rankfunc^{(1)}(\wy) \bigg)}_{\defeq G_1(\wx)} 
\end{align}
\end{subequations}
Alone, this application of dynamic programming does not reduce the complexity from $\bigo{N^2}$ to $\bigo{N}$ as desired because the inner aggregand, $\bigoplus_{y \in\calS_{1}(\wx)} -\rankfunc^{(1)}(\wy) \vspace{2pt}$, itself still takes $\bigo{N}$ time.
However, we are able to compute $\bigoplus_{y \in\calS_{1}(\wx)} -\rankfunc^{(1)}(\wy) \vspace{2pt}$ in amortized $\bigo{1}$ time due to \citeposs[Eq.~1]{fredman-trick} algebraic trick.\looseness=-1

The strategy is to sort\footnote{As before, we take the complexity of sorting to be $\bigo{\seqlen}$ where we can apply radix sort as implemented by Pytorch.} the vertices of the partially ordered structure according to $\rankfunc^{(1)}(\wy) - \rankfunc^{(2)}(\wy)$.
Thus, if we have
$\rankfunc^{(1)}(\wy) - \rankfunc^{(2)}(\wy) < \rankfunc^{(1)}(\wx) - \rankfunc^{(2)}(\wx)$, simple algebra reveals that $\rankfunc^{(2)}(\wx) - \rankfunc^{(2)}(\wy) < \rankfunc^{(1)}(\wx) - \rankfunc^{(1)}(\wy)$.
Thus, for a given $\wx$, every vertex $\wy$ that comes \emph{before} $\wx$ in the sorted order satisfies $\pairFunc(\wx, \wy) = \rankfunc^{(1)}(\wx) -\rankfunc^{(1)}(\wy)$.
Aggregating in this order enables intermediate results to be \emph{reused}. \looseness=-1

\begin{restatable}{algorithm}{efficientcomputation}
    \caption{Computing $G_1$ when $\orderdim=2$.} 
    \begin{algorithmic}[1]
        \Procedure{\textsc{Compute-}$G_1$}{$\rankfunc^{(1)}, \rankfunc^{(2)}, \vertices$}
        \State $\sortedvertices \gets \ttsort \left(\vertices, \ttkey\!=\!\rankfunc^{(1)}\!-\!\rankfunc^{(2)}\right)$ \label{line:alg-1-sort}
        \State $G_1, s_1 \gets \bzero, \bzero$ \Comment{$\bzero$ is the zero element of $\oplus$}
       \For{$n=1$ \textbf{ up to} $N$} \label{line:alg-1-for-1}
        \State $q_1 \!\!= \rankfunc^{(1)}(U_n) + s_1 $ \Comment{$q_1=G_1(U_n)$} \label{line:alg-1-for-2}
        \State $G_1 \, \oplus\!\!= q_1$ \label{line:alg-1-for-3}
        \State $s_1 \, \oplus\!\!= -\rankfunc^{(1)}(U_n)$ \label{line:alg-1-for-4}
       \EndFor
       \State \Return $G_1$
        \EndProcedure
    \end{algorithmic} \label{alg:efficient-computation}
\end{restatable}

Likewise, if we sorted in reverse, i.e., according to $\rankfunc^{(2)}(\wy) - \rankfunc^{(1)}(\wy)$, the same manipulation yields that for a given $\wx$, every vertex $\wy$ that comes \emph{before} $\wx$ in the \emph{reverse} sorted order satisfies $\pairFunc(\wx, \wy) = \rankfunc^{(2)}(\wx) -\rankfunc^{(2)}(\wy)$. \looseness=-1

The algorithm for computing $G_1$ is given in \cref{alg:efficient-computation}, which has $\bigo{\seqlen}$ computations in total. 
Moreover, if parallelized, it can be run in $\bigo{\log \seqlen}$ time.
For $\orderdim > 2$, we speculate that the aggregation algorithm can be done in $\bigo{\orderdim \seqlen \log^{\orderdim-2} \seqlen}$. 
We leave this to future work. 
See \cref{app:dim-gt-2} for further discussion. \looseness=-1

\section{Experiments} \label{sec:experiments}

We report the experimental results on two representative linguistic structure prediction problems in NLP, namely dependency parsing and coreference resolution.
The graph-theoretic definitions of these tasks are given in \cref{ex:coreference,ex:dep-parsing}.
We first convert the linguistic structures to partially ordered (token-split) structures as described in \cref{sec:tss}, and then apply the neural method described in \cref{sec:parameterization} to model the partially ordered structures.\looseness=-1

\subsection{Dependency Parsing} 
\paragraph{Modeling.} 
Orders $\rel$ are not typed in \cref{def:partial-order}.
In other words, under \Cref{def:partial-order}, all relations in a partially ordered structure are of the same type. 
To model dependency type labels, we apply a token-level classifier on the contextualized representation.
During decoding, similar to arc-factored models for dependency parsing, we keep the head word that minimizes $\pairFunc(\wx, \wy)$ for a given $\wx$, i.e., $\argmin_{\wy \in \vertices}  \pairFunc(\wx, \wy)$. \looseness=-1

For pretrained language models, we use \texttt{XLNet-large-cased}\footnote{\url{https://huggingface.co/xlnet-large-cased}} \cite{xlnet} for PTB, \texttt{bert-base-chinese}\footnote{\url{https://huggingface.co/bert-base-chinese}} for CTB, and \texttt{bert-base-multilingual-cased}\footnote{\url{https://huggingface.co/bert-base-multilingual-cased}} for UD.

\paragraph{Datasets.}
We conduct experiments on the English Penn Treebank \cite[PTB;][]{marcus-etal-1993-building}, the Chinese Penn Treebank \cite[CTB;][]{ctb}, and the Universal Dependencies 2.2 \cite[UD;][]{ud2.2}.
Hyperparameter settings and dataset statistics are given in \cref{app:hyperparam-dep,app:dataset-dep}.

\paragraph{Accuracy.} 
We report the experimental results in \cref{tab:ptb-results}. 
The full results on UD are included in \cref{app:ud-extra}.
On PTB and UD, our method achieves state-of-the-art performance compared with $\bigo{\seqlen^3}$ \cite{yang-tu-2022-headed}, $\bigo{\seqlen^2}$ \cite{mrini-etal-2020-rethinking}, and $\bigo{\seqlen}$ \cite{hexa} methods.
Although \citeposs{hexa} method has the same complexity as ours, it is worth noting that our method is more general since it can handle non-projective dependencies \emph{without} using pseudo-projectivization \cite{nivre-nilsson-2005-pseudo}.

\begin{table}[ht]
\centering
\resizebox{\columnwidth}{!}{%
\begin{tabular}{@{}ccccccc@{}}\toprule
  & \multicolumn{2}{c}{PTB} & \multicolumn{2}{c}{CTB} & UD \\
\cmidrule(lr){2-3} \cmidrule(lr){4-5} \cmidrule(lr){6-6}
Model  & UAS & LAS & UAS & LAS & LAS \\ \midrule
\citeauthor{zhou-zhao-2019-head}$^*$         & 97.0  & 95.4  & 91.2  & 89.2 & - \\
\citeauthor{mrini-etal-2020-rethinking}$^*$  & \textbf{97.4}  & {96.3}& 94.6  & 89.3 & - \\ \midrule
\citeauthor{chen-manning-2014-fast}          & 91.8  & 89.6  & 83.9  & 82.4 & - \\
\citeauthor{dozat-biaffine}                  & 95.7  & 94.1  & 89.3  & 88.2 & 91.8 \\
\citeauthor{yang-tu-2022-headed}$^{\#}$         & \textbf{97.4}  & 95.8  & \textbf{93.3} & \textbf{92.3}   & 91.9 \\  
\citeauthor{hexa}     & \textbf{97.4}  & \textbf{96.4}  & 93.2 & 91.9   & 91.8 \\  \midrule
\texttt{Ours} ($\orderdim=2$)   & 97.1 &  96.1 &  90.7  &     89.5 & 91.2 \\
\texttt{Ours} ($\orderdim=4$)   & \textbf{97.4} &  \textbf{96.4} &  92.4  & 91.4  &     \textbf{92.1} \\ \bottomrule
\end{tabular}}
\caption{Experimental results on PTB, CTB, and UD. $^{*}$ indicates usage of extra constituency annotation. $^{\#}$ is our re-implementation using the same pretrained encoder as ours. $\orderdim$ is the dimension of the realizer used.}
\label{tab:ptb-results}
\end{table}

\paragraph{Efficiency.}
\begin{table}[ht]
\centering
\resizebox{\columnwidth}{!}{%
\begin{tabular}{@{}ccccccc@{}}\toprule
 & \multicolumn{3}{c}{Speed (sent/s) $\uparrow$} & \multicolumn{3}{c}{Memory (GB) $\downarrow$} \\
\cmidrule(lr){2-4}
\cmidrule(lr){5-7}
\#token & \texttt{Ours} & \Hexataggertt & \Biaffinett & \texttt{Ours} & \Hexataggertt & \Biaffinett \\ \midrule
32 & 3232 & 2916 & 493 & 1.7 & 2.9 & 4.5 \\
64 & 3332 & 3011 & 328 & 1.7 & 3.0 &  10.1 \\
128 & 3182 & 2649 & 202 & 1.9 & 3.7 &  30.6 \\
256 & 3314 & 3270 & 98 & 3.1 & 4.5  &  56.2 \\ \midrule
 {overall} &  \textbf{3347} & 3176 & 338 & \textbf{1.7} & 3.0 &  10.6 \\ \bottomrule
\end{tabular}}
\caption{Speed and memory consumption comparison on PTB test set. 
We report results averaged over 3 random runs of our method with $\orderdim=2$. The other settings and the results for \Hexataggertt{} and \Biaffinett{} are taken from \citet[Tab. 3]{hexa}.}
\label{tab:efficiency-ptb}
\end{table}

We evaluate the efficiency of our method with two representative baseline models.
As depicted in \cref{tab:efficiency-ptb}, we observe that our method with $\orderdim=2$ is almost 10 times as fast as \Biaffinett{} \cite{dozat-biaffine}, and consumes less memory than \Hexataggertt{} \cite{hexa}, which is $\bigo{\seqlen}$ in space complexity.
We further include some qualitative examples using $\orderdim=2$ in \cref{app:qual-ex} for a more intuitive picture of our method.

\subsection{Coreference Resolution}
\paragraph{Modeling.} 
Our method operates in a two-stage manner to accommodate the two relations in \cref{ex:coreference}. 
First, it extracts a list of entity mentions using the partial order induced by $\labelrel$ (mention relation). In other words, $\wx \rel \wy \iff$ span $\wx\!:\!\wy$ is an entity mention. 
Then, it models the partial order induced by $\labelrelprime$ (coreference relation) over the extracted mentions. 
In other words, $m_1\rel m_2 \iff$ mention $m_1$ corefers to $m_2$. 
To find the optimal coreferent antecedent for each mention $m$, we keep $m'$ that minimizes $\pairFunc(m, m')$.\looseness=-1

The overall complexity of the coreference resolution model is $\bigo{\seqlen}$, since the complexity of the encoder used \cite{Beltagy2020Longformer} and the number of valid mentions are both $\bigo{\seqlen}$, assuming entity mentions are constituents \cite{liu-etal-2022-structured}. 
We experiment on the CoNLL-2012 English shared task dataset \cite[OntoNotes;][]{pradhan-etal-2012-conll}.
Hyperparameter settings and dataset statistics are given in \cref{app:hyperparam-coref,app:dataset-coref}. \looseness=-1

\paragraph{Accuracy.}
The experimental results are displayed in \cref{tab:ontonotes-results}. 
Similar to the results for dependency parsing, an intersection of $2$ total orders can already achieve reasonable performance on coreference resolution.
This provides \emph{empirical evidence} for our assertion in \cref{sec:efficiency-guarantees} that most structures in NLP can be represented as the intersection of at most 2 total orders.
When $\orderdim = 4$, the performance of our method is comparable to \citet{kirstain-etal-2021-coreference}, which uses the same pretrained encoder as ours and requires an $\bigo{\seqlen^2}$ biaffine product computation for token-pair scores. \looseness=-1

 \paragraph{Efficiency.} 
We compare the efficiency of our method with \citeposs{kirstain-etal-2021-coreference} method.
It is worth noting that \citet{kirstain-etal-2021-coreference} has already performed aggressive optimization in both the speed and memory footprint of coreference modeling.
Our method is still 2 times as fast, achieving a speed of 82.8 documents per second vs. 41.9, while using less memory, especially on long documents.
The full efficiency statistics are given in \cref{app:efficiency-eval}. \looseness=-1

\begin{table}[ht]
\begin{center}
\resizebox{\columnwidth}{!}{%
\begin{tabular}{cccc}
\toprule
        & Avg. P        & Avg. R        & Avg. F1       \\ \midrule
\citet{lee-etal-2017-end} & 69.9 & 64.7          & 67.2          \\
\citeauthor{kantor-globerson-2019-coreference} & 76.1 & 77.1          & 76.6          \\
\citet{joshi-etal-2020-spanbert} & 80.1 &   78.9        & 79.6          \\
\citet{xu-choi-2020-revealing}  & 80.3   &  79.5          & 79.9          \\ 
\citet{kirstain-etal-2021-coreference}  & 81.2   &  79.4          & 80.3          \\ \midrule
\texttt{Ours} ($\orderdim=2$)    & 75.2          & 74.8 & 75.0  \\ %
\texttt{Ours} ($\orderdim=4$)    & 79.3          & 79.0 & 79.2 \\  %
\bottomrule %
\end{tabular}
}
\caption{Experimental results on the OntoNotes benchmark. $\orderdim$ is the dimension of the realizer. }
\label{tab:ontonotes-results}
\end{center}%
\end{table}

\section{Related Work\footnote{More related work is included in \cref{app:related-work}.}} 

\subsection{Structured Prediction}
Structured prediction constitutes an important part of natural language processing. 
It involves the modeling of interrelated variables or outputs with structural constraints.
Some representative structured prediction problems are sequence tagging \cite{church-1988-stochastic}, dependency parsing \cite{dep-parsing}, and coreference resolution \cite{discourse-processing}.\looseness=-1

Structured prediction can often be formulated as learning and inference of probabilistic graphical models \cite[\S 2.2]{smith-linguistic-structure}.
The key idea is to represent the probability distribution over the output space using a graph,
in which each vertex corresponds to a random variable, and each edge corresponds to a dependence relation between two random variables. \looseness=-1

\subsection{Graph-Based Parsing}

Graph-based parsers, or arc-factored parsers, construct graphs by scoring all possible arcs \cite{eisner-1996-three,mcdonald-pereira-2006-online} between each pair of words. 
At inference time, they use either maximum spanning tree (MST) finding algorithms \cite{chu-and-liu,edmonds,tarjan}, or the projective MST algorithm \cite{eisner-1996-three} to build the valid dependency trees with maximum score. 
\citet{kiperwasser-goldberg-2016-simple} present a neural graph-based parser that uses the same kind of attention mechanism as \citet{bahdanau-attention} for computing arc scores.
Greedy decoding that independently assigns a head word to each word \cite{dozat-biaffine} is also widely used as an approximation to exact inference algorithms.

\subsection{Tagging-Based Parsing}

Inspired by transition-based parsers \cite{knuth-shift-reduce} and \citeposs{bangalore-joshi-1999-supertagging} seminal work on \defn{supertagging}, one line of work uses pretrained models to parse dependency trees by inferring tags for each word in the input sequence. 
\citet{li-etal-2018-seq2seq} and \citet{kiperwasser-ballesteros-2018-scheduled} predict the relative position of the dependent with respect to its head in a sequence-to-sequence manner. 
\citet{strzyz-etal-2019-viable} give a framework for analyzing similar tagging schemes. 
\citet{gomez-rodriguez-etal-2020-unifying} infer a chunk of actions in a transition-based system for each word in the sequence. \looseness=-1

For non-projective dependency parsing, \citet{gomez-rodriguez-nivre-2010-transition,divisible-transition} show that efficient parsers exist for 2-planar trees \cite{yli2003multiplanarity}, a sub-class of non-projective trees whose arcs can be partitioned into 2 sets while arcs in the same set do not cross each other.
\citet{strzyz-etal-2020-bracketing} propose an encoding scheme for 2-planar trees, enabling a tagging-based parser for such trees.
As mentioned in \cref{sec:motivation-linearization}, to handle the entire set of non-projective trees, the size of the tag set has to be unrestricted, which limits the efficiency and applicability of this series of approaches of approaches.

\subsection{Parsing with Syntactic Distance}

\citet{shen2018neural,shen-etal-2018-straight} introduce a constituent parsing scheme which is also based on the comparison of real numbers.
In this scheme, a neural model is trained to assign one real number, termed the \defn{syntactic distance}, to the gap between every pair of neighboring tokens.
To parse a span into two sub-constituents, the gap with the largest syntactic distance within that span is chosen as the split point.
Parsing can be done by recursively performing the above splitting procedure starting from a given string.
The algorithm has a runtime complexity of $\bigo{\seqlen \log\seqlen}$, which is significantly more efficient than chart-based parsers with $\bigo{\seqlen^2}$ complexity.
However, this method does not generalize easily to perform non-context-free parsing, since it cannot handle the possible discontinuity of constituents. 
Moreover, the recursive splitting procedure restricts the output space of parse trees to be a subset of phrase-structure trees \cite{dyer-syntactic-distance}. \looseness=-1

\section{Conclusion}
In this paper, we propose an order-theoretic treatment of linguistic structured prediction.
Theoretical and empirical results show that most linguistic structure prediction problems can be solved in linear time and memory by framing them as partial orderings of the tokens in the input string.
We demonstrate the effectiveness of our method on dependency parsing and coreference resolution, setting the new state-of-the-art accuracy in some cases and achieving significant efficiency improvements.\looseness=-1

\section{Limitations}

\subsection{Decoding Algorithms}
This work does not provide algorithms for particular structures or algorithms that ensure the well-formedness of structures, such as maximum spanning trees and projective trees.
It remains to be investigated whether existing constrained decoding algorithms for arc-factored models \cite[\emph{inter alia}]{chu-and-liu,edmonds,eisner-1997-bilexical} have their counterparts in the order-theoretic method.
We would like to explore decoding algorithms for structured prediction under order-theoretic formulation in future work.

\subsection{Interpretability}
In our method, the interactions between tokens are not directly modeled as in graph-based structured prediction models, which makes it more difficult to interpret the output of our model.
In addition, we leave to future work the investigation of the total ordering metrics (see \cref{app:qual-ex}) learned by the realizers in our method.

\subsection{Hardness of Learning}
Intuitively, it is harder to learn partial orders over strings than directly modeling the arcs in a graph, since our order-theoretic treatment has much fewer parameters when scoring token pairs. 
We also observed in our experiments that order-theoretic models take more training iterations to converge than arc-factored models. \looseness=-1

For instance, consider the modeling of a tree structure with $\seqlen$ nodes with $\seqlen - 1$ arcs using partial order, which implies $ \seqlen - 1 $ constraints of the form $ \wx \prec \wy$ and $ \seqlen^2 -2\seqlen+1 $ constraints of $ \wx \norel \wy$.  
From a theoretical perspective, $\orderdim = 2$ is sufficient to represent such a structure as shown in \cref{sec:order-theory}. 
In other words, there always exist 2 total orders whose intersection satisfies the aforementioned $\seqlen(\seqlen-1)$ constraints. 
However, it might not be easy to find such orders in practice. \looseness=-1 

A realizer with $\orderdim$ beyond 2 can more easily satisfy the constraints, especially of the form $\wx \norel \wy$---since there are more constraints of this form. 
It allows more possibilities for $\bigvee_{k \in \Nlessthan{\orderdim}} \rankfunck(\wx) \ge \rankfunck(\wy)$ (i.e., more choices of $k$ to satisfy the expression). 
On the other hand, using a small $\orderdim$ might make it harder to satisfy the constraints.

We plan to further investigate the hardness of learning a string partial order in future work.\looseness=-1

\subsection{Precision of floating-point numbers and numerical stability} \label{sec:lim-float}
Our method might be affected by the finite precision of floating-point numbers and numerical instability when applied to very long strings.
Although we did not encounter such issues in our experiments ($\seqlen \le 4096 = 2^{12}$),
issues might arise when $\seqlen > 65536 = 2^{16}$ if \texttt{bfloat16} or half precision is used.
In such extreme cases, our assumption that $\forall k \in \Nlessthan{\orderdim}, \rankfunck$ is injective cannot be fulfilled.
Thus, \emph{not all} totally ordered structures of $\seqlen$ elements can be represented, and our method might not exhibit the desired behavior. \looseness=-1

\section*{Ethics Statement}
We do not believe the work presented here further amplifies biases already present in the datasets and pretrained models. 
Therefore, we foresee no ethical concerns in this work.

\section*{Acknowledgments}
We would like to thank Zhaofeng Wu, Clément Guerner, and Tim Vieira for their invaluable feedback.
We are grateful to the anonymous reviewers for
their insightful comments and suggestions. Afra Amini is supported by ETH AI Center doctoral fellowship.
MS acknowledges support from the Swiss National Science Foundation (Project No. 197155), a Responsible AI grant by the Haslerstiftung; and an ETH Grant (ETH-19 21-1).

\bibliography{anthology,custom}

\begin{thebibliography}{89}
\expandafter\ifx\csname natexlab\endcsname\relax\def\natexlab#1{#1}\fi

\bibitem[{Ajdukiewicz(1935)}]{Ajdukiewicz-35}
Kazimierz Ajdukiewicz. 1935.
\newblock \href
  {https://books.google.com/books/about/Die_Syntaktische_Konnexit%C3%A4t.html?id=pcvBMwEACAAJ}
  {Die syntaktische {K}onnexit\"{a}t}.
\newblock \emph{Studia Philosophica}, 1:1--27.

\bibitem[{Amini and Cotterell(2022)}]{afra-parsing}
Afra Amini and Ryan Cotterell. 2022.
\newblock \href {https://aclanthology.org/2022.emnlp-main.607} {On parsing as
  tagging}.
\newblock In \emph{Proceedings of the 2022 Conference on Empirical Methods in
  Natural Language Processing}, pages 8884--8900, Abu Dhabi, United Arab
  Emirates. Association for Computational Linguistics.

\bibitem[{Amini et~al.(2023)Amini, Liu, and Cotterell}]{hexa}
Afra Amini, Tianyu Liu, and Ryan Cotterell. 2023.
\newblock \href {https://doi.org/10.18653/v1/2023.acl-short.124} {Hexatagging:
  Projective dependency parsing as tagging}.
\newblock In \emph{Proceedings of the 61st Annual Meeting of the Association
  for Computational Linguistics (Volume 2: Short Papers)}, pages 1453--1464,
  Toronto, Canada. Association for Computational Linguistics.

\bibitem[{Athiwaratkun and Wilson(2018)}]{athiwaratkun2018on}
Ben Athiwaratkun and Andrew~Gordon Wilson. 2018.
\newblock \href {https://openreview.net/forum?id=HJCXZQbAZ} {On modeling
  hierarchical data via probabilistic order embeddings}.
\newblock In \emph{International Conference on Learning Representations}.

\bibitem[{Bahdanau et~al.(2015)Bahdanau, Cho, and Bengio}]{bahdanau-attention}
Dzmitry Bahdanau, {Kyung Hyun} Cho, and Yoshua Bengio. 2015.
\newblock \href {https://arxiv.org/abs/1409.0473} {Neural machine translation
  by jointly learning to align and translate}.
\newblock In \emph{3rd International Conference on Learning Representations,
  ICLR 2015}.

\bibitem[{Baker et~al.(1972)Baker, Fishburn, and Roberts}]{baker-1972}
K.~A. Baker, P.~C. Fishburn, and F.~S. Roberts. 1972.
\newblock \href {https://doi.org/https://doi.org/10.1002/net.3230020103}
  {Partial orders of dimension 2}.
\newblock \emph{Networks}, 2(1):11--28.

\bibitem[{Bangalore and Joshi(1999)}]{bangalore-joshi-1999-supertagging}
Srinivas Bangalore and Aravind~K. Joshi. 1999.
\newblock \href {https://aclanthology.org/J99-2004} {{S}upertagging: An
  approach to almost parsing}.
\newblock \emph{Computational Linguistics}, 25(2):237--265.

\bibitem[{Bar-Hillel(1953)}]{bar-hillel-ccg}
Yehoshua Bar-Hillel. 1953.
\newblock \href {http://www.jstor.org/stable/410452} {A quasi-arithmetical
  notation for syntactic description}.
\newblock \emph{Language}, 29(1):47--58.

\bibitem[{Beltagy et~al.(2020)Beltagy, Peters, and
  Cohan}]{Beltagy2020Longformer}
Iz~Beltagy, Matthew~E. Peters, and Arman Cohan. 2020.
\newblock \href {https://arxiv.org/abs/2004.05150} {Longformer: The
  long-document transformer}.
\newblock \emph{arXiv:2004.05150}.

\bibitem[{Bentley(1979)}]{range-tree-bentley}
Jon~Louis Bentley. 1979.
\newblock \href {https://doi.org/https://doi.org/10.1016/0020-0190(79)90117-0}
  {Decomposable searching problems}.
\newblock \emph{Information Processing Letters}, 8(5):244--251.

\bibitem[{Bentley(1980)}]{multidim-bentley}
Jon~Louis Bentley. 1980.
\newblock \href {https://doi.org/10.1145/358841.358850} {Multidimensional
  divide-and-conquer}.
\newblock \emph{Commun. ACM}, 23(4):214–229.

\bibitem[{Berg et~al.(2008)Berg, Cheong, Kreveld, and Overmars}]{comp-geometry}
Mark~de Berg, Otfried Cheong, Marc~van Kreveld, and Mark Overmars. 2008.
\newblock \href {https://dl.acm.org/doi/10.5555/1370949} {\emph{Computational
  Geometry: Algorithms and Applications}}, 3rd edition.
\newblock Springer-Verlag TELOS, Santa Clara, CA, USA.

\bibitem[{Birkhoff(1967)}]{birkhoff-1967-lattice}
G.~Birkhoff. 1967.
\newblock \href {https://books.google.ch/books?id=yiNlzwEACAAJ} {\emph{Lattice
  Theory}}.
\newblock American Mathematical Society colloquium publications. American
  Mathematical Society.

\bibitem[{Cayley(1889)}]{cayley-trees}
Arthur Cayley. 1889.
\newblock \href {https://doi.org/10.1017/CBO9780511703799.010} {A theorem on
  trees}.
\newblock \emph{Quarterly Journal of Mathematics}, 23:376--378.

\bibitem[{Chazelle(1988)}]{semigroup-chazelle}
Bernard Chazelle. 1988.
\newblock \href {https://doi.org/10.1137/0217026} {A functional approach to
  data structures and its use in multidimensional searching}.
\newblock \emph{SIAM Journal on Computing}, 17(3):427--462.

\bibitem[{Chazelle(1990{\natexlab{a}})}]{range-tree-1-chazelle}
Bernard Chazelle. 1990{\natexlab{a}}.
\newblock \href {https://doi.org/10.1145/77600.77614} {Lower bounds for
  orthogonal range searching: {I}. {T}he reporting case}.
\newblock \emph{Journal of the ACM}, 37(2):200–212.

\bibitem[{Chazelle(1990{\natexlab{b}})}]{range-tree-2-chazelle}
Bernard Chazelle. 1990{\natexlab{b}}.
\newblock \href {https://doi.org/10.1145/79147.79149} {Lower bounds for
  orthogonal range searching: Part {II}. {T}he arithmetic model}.
\newblock \emph{Journal of the ACM}, 37(3):439–463.

\bibitem[{Chen and Manning(2014)}]{chen-manning-2014-fast}
Danqi Chen and Christopher Manning. 2014.
\newblock \href {https://doi.org/10.3115/v1/D14-1082} {A fast and accurate
  dependency parser using neural networks}.
\newblock In \emph{Proceedings of the 2014 Conference on Empirical Methods in
  Natural Language Processing ({EMNLP})}, pages 740--750, Doha, Qatar.
  Association for Computational Linguistics.

\bibitem[{Chomsky(2015)}]{minimalist}
Noam Chomsky. 2015.
\newblock \href {http://www.jstor.org/stable/j.ctt17kk8xd} {\emph{The
  Minimalist Program}}, 20 edition.
\newblock The MIT Press.

\bibitem[{Chu and Liu(1965)}]{chu-and-liu}
Yoeng-Jin Chu and Tseng-Hong Liu. 1965.
\newblock \href {https://cir.nii.ac.jp/crid/1570854175817997952?lang=en} {On
  the shortest arborescence of a directed graph}.
\newblock \emph{Scientia Sinica}, 14:1396--1400.

\bibitem[{Church(1988)}]{church-1988-stochastic}
Kenneth~Ward Church. 1988.
\newblock \href {https://doi.org/10.3115/974235.974260} {A stochastic parts
  program and noun phrase parser for unrestricted text}.
\newblock In \emph{Second Conference on Applied Natural Language Processing},
  pages 136--143, Austin, Texas, USA. Association for Computational
  Linguistics.

\bibitem[{de~Marneffe and Manning(2008)}]{de2008stanford}
Marie-Catherine de~Marneffe and Christopher~D. Manning. 2008.
\newblock \href
  {https://downloads.cs.stanford.edu/nlp/software/dependencies_manual.pdf}
  {{S}tanford typed dependencies manual}.
\newblock Technical report, Technical report, Stanford University.

\bibitem[{Dozat and Manning(2017)}]{dozat-biaffine}
Timothy Dozat and Christopher~D. Manning. 2017.
\newblock \href {https://openreview.net/forum?id=Hk95PK9le} {Deep biaffine
  attention for neural dependency parsing}.
\newblock In \emph{5th International Conference on Learning Representations,
  {ICLR} 2017, Toulon, France, April 24-26, 2017, Conference Track
  Proceedings}.

\bibitem[{Dushnik and Miller(1941)}]{poset-dushnik-miller}
Ben Dushnik and E.~W. Miller. 1941.
\newblock \href {http://www.jstor.org/stable/2371374} {Partially ordered sets}.
\newblock \emph{American Journal of Mathematics}, 63(3):600--610.

\bibitem[{Dyer et~al.(2019)Dyer, Melis, and Blunsom}]{dyer-syntactic-distance}
Chris Dyer, G{\'{a}}bor Melis, and Phil Blunsom. 2019.
\newblock \href {http://arxiv.org/abs/1909.09428} {A critical analysis of
  biased parsers in unsupervised parsing}.
\newblock \emph{CoRR}, abs/1909.09428.

\bibitem[{Edmonds(1967)}]{edmonds}
Jack Edmonds. 1967.
\newblock \href
  {https://nvlpubs.nist.gov/nistpubs/jres/71b/jresv71bn4p233_a1b.pdf} {Optimum
  branchings}.
\newblock \emph{Journal of Research of the national Bureau of Standards B},
  71(4):233--240.

\bibitem[{Eisner(1997)}]{eisner-1997-bilexical}
Jason Eisner. 1997.
\newblock \href {https://aclanthology.org/1997.iwpt-1.10} {Bilexical grammars
  and a cubic-time probabilistic parser}.
\newblock In \emph{Proceedings of the Fifth International Workshop on Parsing
  Technologies}, pages 54--65, Boston/Cambridge, Massachusetts, USA.
  Association for Computational Linguistics.

\bibitem[{Eisner(1996)}]{eisner-1996-three}
Jason~M. Eisner. 1996.
\newblock \href {https://aclanthology.org/C96-1058} {Three new probabilistic
  models for dependency parsing: An exploration}.
\newblock In \emph{{COLING} 1996 Volume 1: The 16th International Conference on
  Computational Linguistics}.

\bibitem[{Fredman(1976)}]{fredman-trick}
Michael~L. Fredman. 1976.
\newblock \href {https://doi.org/10.1137/0205006} {New bounds on the complexity
  of the shortest path problem}.
\newblock \emph{SIAM Journal on Computing}, 5(1):83--89.

\bibitem[{G{\'o}mez-Rodr{\'\i}guez et~al.(2020)G{\'o}mez-Rodr{\'\i}guez,
  Strzyz, and Vilares}]{gomez-rodriguez-etal-2020-unifying}
Carlos G{\'o}mez-Rodr{\'\i}guez, Michalina Strzyz, and David Vilares. 2020.
\newblock \href {https://doi.org/10.18653/v1/2020.coling-main.336} {A unifying
  theory of transition-based and sequence labeling parsing}.
\newblock In \emph{Proceedings of the 28th International Conference on
  Computational Linguistics}, pages 3776--3793, Barcelona, Spain (Online).
  International Committee on Computational Linguistics.

\bibitem[{G{\'o}mez-Rodr{\'\i}guez and
  Vilares(2018)}]{gomez-rodriguez-vilares-2018-constituent}
Carlos G{\'o}mez-Rodr{\'\i}guez and David Vilares. 2018.
\newblock \href {https://doi.org/10.18653/v1/D18-1162} {Constituent parsing as
  sequence labeling}.
\newblock In \emph{Proceedings of the 2018 Conference on Empirical Methods in
  Natural Language Processing}, pages 1314--1324, Brussels, Belgium.
  Association for Computational Linguistics.

\bibitem[{Gómez-Rodríguez and
  Nivre(2010)}]{gomez-rodriguez-nivre-2010-transition}
Carlos Gómez-Rodríguez and Joakim Nivre. 2010.
\newblock \href {https://aclanthology.org/P10-1151} {A transition-based parser
  for 2-planar dependency structures}.
\newblock In \emph{Proceedings of the 48th Annual Meeting of the Association
  for Computational Linguistics}, pages 1492--1501, Uppsala, Sweden.
  Association for Computational Linguistics.

\bibitem[{Gómez-Rodríguez and Nivre(2013)}]{divisible-transition}
Carlos Gómez-Rodríguez and Joakim Nivre. 2013.
\newblock \href {https://doi.org/10.1162/COLI_a_00150} {{Divisible Transition
  Systems and Multiplanar Dependency Parsing }}.
\newblock \emph{Computational Linguistics}, 39(4):799--845.

\bibitem[{Hausdorff(1914)}]{hausdorff1914grundzüge}
F.~Hausdorff. 1914.
\newblock \href {https://books.google.com/books?id=KTs4AAAAMAAJ}
  {\emph{Grundz{\"u}ge der Mengenlehre}}.
\newblock G{\"o}schens Lehrb{\"u}cherei/Gruppe I: Reine und Angewandte
  Mathematik Series. Veit \& Company.

\bibitem[{Hiraguchi(1955)}]{hiraguchi-55}
Toshio Hiraguchi. 1955.
\newblock \href {http://hdl.handle.net/2297/33759} {On the dimension of
  orders}.
\newblock \emph{The Science Reports of the Kanazawa University}, 4:1--20.

\bibitem[{Hochreiter and Schmidhuber(1997)}]{lstm}
Sepp Hochreiter and Jürgen Schmidhuber. 1997.
\newblock \href {https://doi.org/10.1162/neco.1997.9.8.1735} {{Long Short-Term
  Memory}}.
\newblock \emph{Neural Computation}, 9(8):1735--1780.

\bibitem[{Joshi et~al.(2020)Joshi, Chen, Liu, Weld, Zettlemoyer, and
  Levy}]{joshi-etal-2020-spanbert}
Mandar Joshi, Danqi Chen, Yinhan Liu, Daniel~S. Weld, Luke Zettlemoyer, and
  Omer Levy. 2020.
\newblock \href {https://doi.org/10.1162/tacl_a_00300} {{S}pan{BERT}: Improving
  pre-training by representing and predicting spans}.
\newblock \emph{Transactions of the Association for Computational Linguistics},
  8:64--77.

\bibitem[{Kantor and Globerson(2019)}]{kantor-globerson-2019-coreference}
Ben Kantor and Amir Globerson. 2019.
\newblock \href {https://doi.org/10.18653/v1/P19-1066} {Coreference resolution
  with entity equalization}.
\newblock In \emph{Proceedings of the 57th Annual Meeting of the Association
  for Computational Linguistics}, pages 673--677, Florence, Italy. Association
  for Computational Linguistics.

\bibitem[{Kiperwasser and
  Ballesteros(2018)}]{kiperwasser-ballesteros-2018-scheduled}
Eliyahu Kiperwasser and Miguel Ballesteros. 2018.
\newblock \href {https://doi.org/10.1162/tacl_a_00017} {Scheduled multi-task
  learning: From syntax to translation}.
\newblock \emph{Transactions of the Association for Computational Linguistics},
  6:225--240.

\bibitem[{Kiperwasser and Goldberg(2016)}]{kiperwasser-goldberg-2016-simple}
Eliyahu Kiperwasser and Yoav Goldberg. 2016.
\newblock \href {https://doi.org/10.1162/tacl_a_00101} {Simple and accurate
  dependency parsing using bidirectional {LSTM} feature representations}.
\newblock \emph{Transactions of the Association for Computational Linguistics},
  4:313--327.

\bibitem[{Kirstain et~al.(2021)Kirstain, Ram, and
  Levy}]{kirstain-etal-2021-coreference}
Yuval Kirstain, Ori Ram, and Omer Levy. 2021.
\newblock \href {https://doi.org/10.18653/v1/2021.acl-short.3} {Coreference
  resolution without span representations}.
\newblock In \emph{Proceedings of the 59th Annual Meeting of the Association
  for Computational Linguistics and the 11th International Joint Conference on
  Natural Language Processing (Volume 2: Short Papers)}, pages 14--19, Online.
  Association for Computational Linguistics.

\bibitem[{Kitaev and Klein(2020)}]{kitaev-klein-2020-tetra}
Nikita Kitaev and Dan Klein. 2020.
\newblock \href {https://doi.org/10.18653/v1/2020.acl-main.557} {Tetra-tagging:
  Word-synchronous parsing with linear-time inference}.
\newblock In \emph{Proceedings of the 58th Annual Meeting of the Association
  for Computational Linguistics}, pages 6255--6261, Online. Association for
  Computational Linguistics.

\bibitem[{Knuth(1965)}]{knuth-shift-reduce}
Donald~E. Knuth. 1965.
\newblock \href {https://doi.org/https://doi.org/10.1016/S0019-9958(65)90426-2}
  {On the translation of languages from left to right}.
\newblock \emph{Information and Control}, 8(6):607--639.

\bibitem[{Knuth(1997)}]{knuth-comp-programming}
Donald~E. Knuth. 1997.
\newblock \href {https://dl.acm.org/doi/book/10.5555/260999} {\emph{The Art of
  Computer Programming: Fundamental Algorithms}}, 3 edition, volume~1.
\newblock Addison Wesley Longman Publishing Co., Inc., USA.

\bibitem[{K\"ubler et~al.(2009)K\"ubler, McDonald, and Nivre}]{dep-parsing}
Sandra K\"ubler, Ryan McDonald, and Joakim Nivre. 2009.
\newblock \href
  {http://www.morganclaypool.com/doi/abs/10.2200/S00169ED1V01Y200901HLT002}
  {\emph{Dependency Parsing}}.
\newblock Springer Cham.

\bibitem[{Lawler(1978)}]{lawler-trees-series-parallel}
E.~L. Lawler. 1978.
\newblock \href {https://doi.org/https://doi.org/10.1016/S0167-5060(08)70323-6}
  {Sequencing jobs to minimize total weighted completion time subject to
  precedence constraints}.
\newblock In B.~Alspach, P.~Hell, and D.J. Miller, editors, \emph{Algorithmic
  Aspects of Combinatorics}, volume~2 of \emph{Annals of Discrete Mathematics},
  pages 75--90. Elsevier.

\bibitem[{Lee et~al.(2017)Lee, He, Lewis, and Zettlemoyer}]{lee-etal-2017-end}
Kenton Lee, Luheng He, Mike Lewis, and Luke Zettlemoyer. 2017.
\newblock \href {https://doi.org/10.18653/v1/D17-1018} {End-to-end neural
  coreference resolution}.
\newblock In \emph{Proceedings of the 2017 Conference on Empirical Methods in
  Natural Language Processing}, pages 188--197, Copenhagen, Denmark.
  Association for Computational Linguistics.

\bibitem[{Li et~al.(2018)Li, Cai, He, and Zhao}]{li-etal-2018-seq2seq}
Zuchao Li, Jiaxun Cai, Shexia He, and Hai Zhao. 2018.
\newblock \href {https://aclanthology.org/C18-1271} {Seq2seq dependency
  parsing}.
\newblock In \emph{Proceedings of the 27th International Conference on
  Computational Linguistics}, pages 3203--3214, Santa Fe, New Mexico, USA.
  Association for Computational Linguistics.

\bibitem[{Liu et~al.(2022)Liu, Jiang, Cotterell, and
  Sachan}]{liu-etal-2022-structured}
Tianyu Liu, Yuchen Jiang, Ryan Cotterell, and Mrinmaya Sachan. 2022.
\newblock \href {https://doi.org/10.18653/v1/2022.naacl-main.189} {A structured
  span selector}.
\newblock In \emph{Proceedings of the 2022 Conference of the North American
  Chapter of the Association for Computational Linguistics: Human Language
  Technologies}, pages 2629--2641, Seattle, United States. Association for
  Computational Linguistics.

\bibitem[{Loshchilov and Hutter(2019)}]{loshchilov2018decoupled}
Ilya Loshchilov and Frank Hutter. 2019.
\newblock \href {https://openreview.net/forum?id=Bkg6RiCqY7} {Decoupled weight
  decay regularization}.
\newblock In \emph{International Conference on Learning Representations}.

\bibitem[{MacNeille(1937)}]{macneille-1937}
Holbrook~Mann MacNeille. 1937.
\newblock \href
  {https://www.ams.org/journals/tran/1937-042-03/S0002-9947-1937-1501929-X/S0002-9947-1937-1501929-X.pdf}
  {Partially ordered sets}.
\newblock \emph{Transactions of the American Mathematical Society},
  42(3):416--460.

\bibitem[{Marcus et~al.(1993)Marcus, Santorini, and
  Marcinkiewicz}]{marcus-etal-1993-building}
Mitchell~P. Marcus, Beatrice Santorini, and Mary~Ann Marcinkiewicz. 1993.
\newblock \href {https://aclanthology.org/J93-2004} {Building a large annotated
  corpus of {E}nglish: The {P}enn {T}reebank}.
\newblock \emph{Computational Linguistics}, 19(2):313--330.

\bibitem[{McCullagh(1980)}]{mccullagh-ordinal-regression}
Peter McCullagh. 1980.
\newblock \href {http://www.jstor.org/stable/2984952} {Regression models for
  ordinal data}.
\newblock \emph{Journal of the Royal Statistical Society. Series B
  (Methodological)}, 42(2):109--142.

\bibitem[{McDonald and Pereira(2006)}]{mcdonald-pereira-2006-online}
Ryan McDonald and Fernando Pereira. 2006.
\newblock \href {https://aclanthology.org/E06-1011} {Online learning of
  approximate dependency parsing algorithms}.
\newblock In \emph{11th Conference of the {E}uropean Chapter of the Association
  for Computational Linguistics}, pages 81--88, Trento, Italy. Association for
  Computational Linguistics.

\bibitem[{McDonald et~al.(2005)McDonald, Pereira, Ribarov, and
  Haji{\v{c}}}]{mcdonald-etal-2005-non}
Ryan McDonald, Fernando Pereira, Kiril Ribarov, and Jan Haji{\v{c}}. 2005.
\newblock \href {https://aclanthology.org/H05-1066} {Non-projective dependency
  parsing using spanning tree algorithms}.
\newblock In \emph{Proceedings of Human Language Technology Conference and
  Conference on Empirical Methods in Natural Language Processing}, pages
  523--530, Vancouver, British Columbia, Canada. Association for Computational
  Linguistics.

\bibitem[{Miller(1994)}]{miller-1994-wordnet}
George~A. Miller. 1994.
\newblock \href {https://aclanthology.org/H94-1111} {{W}ord{N}et: A lexical
  database for {E}nglish}.
\newblock In \emph{{H}uman {L}anguage {T}echnology: Proceedings of a Workshop
  held at {P}lainsboro, {N}ew {J}ersey, {M}arch 8-11, 1994}.

\bibitem[{Mrini et~al.(2020)Mrini, Dernoncourt, Tran, Bui, Chang, and
  Nakashole}]{mrini-etal-2020-rethinking}
Khalil Mrini, Franck Dernoncourt, Quan~Hung Tran, Trung Bui, Walter Chang, and
  Ndapa Nakashole. 2020.
\newblock \href {https://doi.org/10.18653/v1/2020.findings-emnlp.65}
  {Rethinking self-attention: Towards interpretability in neural parsing}.
\newblock In \emph{Findings of the Association for Computational Linguistics:
  EMNLP 2020}, pages 731--742, Online. Association for Computational
  Linguistics.

\bibitem[{Nivre(2003)}]{nivre-2003-efficient}
Joakim Nivre. 2003.
\newblock \href {https://aclanthology.org/W03-3017} {An efficient algorithm for
  projective dependency parsing}.
\newblock In \emph{Proceedings of the Eighth International Conference on
  Parsing Technologies}, pages 149--160, Nancy, France.

\bibitem[{Nivre et~al.(2018)Nivre, Abrams, Agi{\'c}, Ahrenberg, Antonsen,
  Aranzabe, Arutie, Asahara, Ateyah, Attia, Atutxa, Augustinus, Badmaeva,
  Ballesteros, Banerjee, Bank, Barbu~Mititelu, Bauer, Bellato, Bengoetxea,
  Bhat, Biagetti, Bick, Blokland, Bobicev, B{\"o}rstell, Bosco, Bouma, Bowman,
  Boyd, Burchardt, Candito, Caron, Caron, Cebiro{\u g}lu~Eryi{\u g}it, Celano,
  Cetin, Chalub, Choi, Cho, Chun, Cinkov{\'a}, Collomb, {\c C}{\"o}ltekin,
  Connor, Courtin, Davidson, de~Marneffe, de~Paiva, Diaz~de Ilarraza,
  Dickerson, Dirix, Dobrovoljc, Dozat, Droganova, Dwivedi, Eli, Elkahky,
  Ephrem, Erjavec, Etienne, Farkas, Fernandez~Alcalde, Foster, Freitas,
  Gajdo{\v s}ov{\'a}, Galbraith, Garcia, G{\"a}rdenfors, Gerdes, Ginter,
  Goenaga, Gojenola, G{\"o}k{\i}rmak, Goldberg, G{\'o}mez~Guinovart,
  Gonz{\'a}les~Saavedra, Grioni, Gr{\= u}z{\={\i}}tis, Guillaume,
  Guillot-Barbance, Habash, Haji{\v c}, Haji{\v c}~jr., H{\`a}~M{\~y}, Han,
  Harris, Haug, Hladk{\'a}, Hlav{\'a}{\v c}ov{\'a}, Hociung, Hohle, Hwang, Ion,
  Irimia, Jel{\'{\i}}nek, Johannsen, J{\o}rgensen, Ka{\c s}{\i}kara, Kahane,
  Kanayama, Kanerva, Kayadelen, Kettnerov{\'a}, Kirchner, Kotsyba, Krek, Kwak,
  Laippala, Lambertino, Lando, Larasati, Lavrentiev, Lee,
  L{\^e}~H{\fontencoding{T5}\selectfont\`{\^o}}ng, Lenci, Lertpradit, Leung,
  Li, Li, Li, Lim, Ljube{\v s}i{\'c}, Loginova, Lyashevskaya, Lynn, Macketanz,
  Makazhanov, Mandl, Manning, Manurung, M{\u a}r{\u a}nduc, Mare{\v c}ek,
  Marheinecke, Mart{\'{\i}}nez~Alonso, Martins, Ma{\v s}ek, Matsumoto,
  {McDonald}, Mendon{\c c}a, Miekka, Missil{\"a}, Mititelu, Miyao, Montemagni,
  More, Moreno~Romero, Mori, Mortensen, Moskalevskyi, Muischnek, Murawaki,
  M{\"u}{\"u}risep, Nainwani, Navarro~Hor{\~n}iacek, Nedoluzhko, Ne{\v
  s}pore-B{\=e}rzkalne, Nguy{\~{\^e}}n~Th{\d i}, Nguy{\~{\^e}}n Th{\d i}~Minh,
  Nikolaev, Nitisaroj, Nurmi, Ojala, Ol{\'u}{\`o}kun, Omura, Osenova,
  {\"O}stling, {\O}vrelid, Partanen, Pascual, Passarotti, Patejuk, Peng, Perez,
  Perrier, Petrov, Piitulainen, Pitler, Plank, Poibeau, Popel, Pretkalni{\c
  n}a, Pr{\'e}vost, Prokopidis, Przepi{\'o}rkowski, Puolakainen, Pyysalo,
  R{\"a}{\"a}bis, Rademaker, Ramasamy, Rama, Ramisch, Ravishankar, Real, Reddy,
  Rehm, Rie{\ss}ler, Rinaldi, Rituma, Rocha, Romanenko, Rosa, Rovati, Roșca,
  Rudina, Sadde, Saleh, Samard{\v z}i{\'c}, Samson, Sanguinetti,
  Saul{\={\i}}te, Sawanakunanon, Schneider, Schuster, Seddah, Seeker, Seraji,
  Shen, Shimada, Shohibussirri, Sichinava, Silveira, Simi, Simionescu,
  Simk{\'o}, {\v S}imkov{\'a}, Simov, Smith, Soares-Bastos, Stella, Straka,
  Strnadov{\'a}, Suhr, Sulubacak, Sz{\'a}nt{\'o}, Taji, Takahashi, Tanaka,
  Tellier, Trosterud, Trukhina, Tsarfaty, Tyers, Uematsu, Ure{\v s}ov{\'a},
  Uria, Uszkoreit, Vajjala, van Niekerk, van Noord, Varga, Vincze, Wallin,
  Washington, Williams, Wir{\'e}n, Woldemariam, Wong, Yan, Yavrumyan, Yu, {\v
  Z}abokrtsk{\'y}, Zeldes, Zeman, Zhang, and Zhu}]{ud2.2}
Joakim Nivre, Mitchell Abrams, {\v Z}eljko Agi{\'c}, Lars Ahrenberg, Lene
  Antonsen, Maria~Jesus Aranzabe, Gashaw Arutie, Masayuki Asahara, Luma Ateyah,
  Mohammed Attia, Aitziber Atutxa, Liesbeth Augustinus, Elena Badmaeva, Miguel
  Ballesteros, Esha Banerjee, Sebastian Bank, Verginica Barbu~Mititelu, John
  Bauer, Sandra Bellato, Kepa Bengoetxea, Riyaz~Ahmad Bhat, Erica Biagetti,
  Eckhard Bick, Rogier Blokland, Victoria Bobicev, Carl B{\"o}rstell, Cristina
  Bosco, Gosse Bouma, Sam Bowman, Adriane Boyd, Aljoscha Burchardt, Marie
  Candito, Bernard Caron, Gauthier Caron, G{\"u}l{\c s}en Cebiro{\u
  g}lu~Eryi{\u g}it, Giuseppe G.~A. Celano, Savas Cetin, Fabricio Chalub, Jinho
  Choi, Yongseok Cho, Jayeol Chun, Silvie Cinkov{\'a}, Aur{\'e}lie Collomb, {\c
  C}a{\u g}r{\i} {\c C}{\"o}ltekin, Miriam Connor, Marine Courtin, Elizabeth
  Davidson, Marie-Catherine de~Marneffe, Valeria de~Paiva, Arantza Diaz~de
  Ilarraza, Carly Dickerson, Peter Dirix, Kaja Dobrovoljc, Timothy Dozat, Kira
  Droganova, Puneet Dwivedi, Marhaba Eli, Ali Elkahky, Binyam Ephrem, Toma{\v
  z} Erjavec, Aline Etienne, Rich{\'a}rd Farkas, Hector Fernandez~Alcalde,
  Jennifer Foster, Cl{\'a}udia Freitas, Katar{\'{\i}}na Gajdo{\v s}ov{\'a},
  Daniel Galbraith, Marcos Garcia, Moa G{\"a}rdenfors, Kim Gerdes, Filip
  Ginter, Iakes Goenaga, Koldo Gojenola, Memduh G{\"o}k{\i}rmak, Yoav Goldberg,
  Xavier G{\'o}mez~Guinovart, Berta Gonz{\'a}les~Saavedra, Matias Grioni,
  Normunds Gr{\= u}z{\={\i}}tis, Bruno Guillaume, C{\'e}line Guillot-Barbance,
  Nizar Habash, Jan Haji{\v c}, Jan Haji{\v c}~jr., Linh H{\`a}~M{\~y}, Na-Rae
  Han, Kim Harris, Dag Haug, Barbora Hladk{\'a}, Jaroslava Hlav{\'a}{\v
  c}ov{\'a}, Florinel Hociung, Petter Hohle, Jena Hwang, Radu Ion, Elena
  Irimia, Tom{\'a}{\v s} Jel{\'{\i}}nek, Anders Johannsen, Fredrik
  J{\o}rgensen, H{\"u}ner Ka{\c s}{\i}kara, Sylvain Kahane, Hiroshi Kanayama,
  Jenna Kanerva, Tolga Kayadelen, V{\'a}clava Kettnerov{\'a}, Jesse Kirchner,
  Natalia Kotsyba, Simon Krek, Sookyoung Kwak, Veronika Laippala, Lorenzo
  Lambertino, Tatiana Lando, Septina~Dian Larasati, Alexei Lavrentiev, John
  Lee, Ph{\fontencoding{T5}\selectfont ươ}ng
  L{\^e}~H{\fontencoding{T5}\selectfont\`{\^o}}ng, Alessandro Lenci, Saran
  Lertpradit, Herman Leung, Cheuk~Ying Li, Josie Li, Keying Li, {KyungTae} Lim,
  Nikola Ljube{\v s}i{\'c}, Olga Loginova, Olga Lyashevskaya, Teresa Lynn,
  Vivien Macketanz, Aibek Makazhanov, Michael Mandl, Christopher Manning, Ruli
  Manurung, C{\u a}t{\u a}lina M{\u a}r{\u a}nduc, David Mare{\v c}ek, Katrin
  Marheinecke, H{\'e}ctor Mart{\'{\i}}nez~Alonso, Andr{\'e} Martins, Jan Ma{\v
  s}ek, Yuji Matsumoto, Ryan {McDonald}, Gustavo Mendon{\c c}a, Niko Miekka,
  Anna Missil{\"a}, C{\u a}t{\u a}lin Mititelu, Yusuke Miyao, Simonetta
  Montemagni, Amir More, Laura Moreno~Romero, Shinsuke Mori, Bjartur Mortensen,
  Bohdan Moskalevskyi, Kadri Muischnek, Yugo Murawaki, Kaili M{\"u}{\"u}risep,
  Pinkey Nainwani, Juan~Ignacio Navarro~Hor{\~n}iacek, Anna Nedoluzhko, Gunta
  Ne{\v s}pore-B{\=e}rzkalne, L{\fontencoding{T5}\selectfont ươ}ng
  Nguy{\~{\^e}}n~Th{\d i}, Huy{\fontencoding{T5}\selectfont\`{\^e}}n
  Nguy{\~{\^e}}n Th{\d i}~Minh, Vitaly Nikolaev, Rattima Nitisaroj, Hanna
  Nurmi, Stina Ojala, Ad{\'e}day{\d {\`o}} Ol{\'u}{\`o}kun, Mai Omura, Petya
  Osenova, Robert {\"O}stling, Lilja {\O}vrelid, Niko Partanen, Elena Pascual,
  Marco Passarotti, Agnieszka Patejuk, Siyao Peng, Cenel-Augusto Perez, Guy
  Perrier, Slav Petrov, Jussi Piitulainen, Emily Pitler, Barbara Plank, Thierry
  Poibeau, Martin Popel, Lauma Pretkalni{\c n}a, Sophie Pr{\'e}vost, Prokopis
  Prokopidis, Adam Przepi{\'o}rkowski, Tiina Puolakainen, Sampo Pyysalo,
  Andriela R{\"a}{\"a}bis, Alexandre Rademaker, Loganathan Ramasamy, Taraka
  Rama, Carlos Ramisch, Vinit Ravishankar, Livy Real, Siva Reddy, Georg Rehm,
  Michael Rie{\ss}ler, Larissa Rinaldi, Laura Rituma, Luisa Rocha, Mykhailo
  Romanenko, Rudolf Rosa, Davide Rovati, Valentin Roșca, Olga Rudina, Shoval
  Sadde, Shadi Saleh, Tanja Samard{\v z}i{\'c}, Stephanie Samson, Manuela
  Sanguinetti, Baiba Saul{\={\i}}te, Yanin Sawanakunanon, Nathan Schneider,
  Sebastian Schuster, Djam{\'e} Seddah, Wolfgang Seeker, Mojgan Seraji,
  Mo~Shen, Atsuko Shimada, Muh Shohibussirri, Dmitry Sichinava, Natalia
  Silveira, Maria Simi, Radu Simionescu, Katalin Simk{\'o}, M{\'a}ria {\v
  S}imkov{\'a}, Kiril Simov, Aaron Smith, Isabela Soares-Bastos, Antonio
  Stella, Milan Straka, Jana Strnadov{\'a}, Alane Suhr, Umut Sulubacak, Zsolt
  Sz{\'a}nt{\'o}, Dima Taji, Yuta Takahashi, Takaaki Tanaka, Isabelle Tellier,
  Trond Trosterud, Anna Trukhina, Reut Tsarfaty, Francis Tyers, Sumire Uematsu,
  Zde{\v n}ka Ure{\v s}ov{\'a}, Larraitz Uria, Hans Uszkoreit, Sowmya Vajjala,
  Daniel van Niekerk, Gertjan van Noord, Viktor Varga, Veronika Vincze, Lars
  Wallin, Jonathan~North Washington, Seyi Williams, Mats Wir{\'e}n, Tsegay
  Woldemariam, Tak-sum Wong, Chunxiao Yan, Marat~M. Yavrumyan, Zhuoran Yu,
  Zden{\v e}k {\v Z}abokrtsk{\'y}, Amir Zeldes, Daniel Zeman, Manying Zhang,
  and Hanzhi Zhu. 2018.
\newblock \href {http://hdl.handle.net/11234/1-2837} {Universal dependencies
  2.2}.
\newblock {LINDAT}/{CLARIAH}-{CZ} digital library at the Institute of Formal
  and Applied Linguistics ({{\'U}FAL}), Faculty of Mathematics and Physics,
  Charles University.

\bibitem[{Nivre and Nilsson(2005)}]{nivre-nilsson-2005-pseudo}
Joakim Nivre and Jens Nilsson. 2005.
\newblock \href {https://doi.org/10.3115/1219840.1219853} {Pseudo-projective
  dependency parsing}.
\newblock In \emph{Proceedings of the 43rd Annual Meeting of the Association
  for Computational Linguistics ({ACL}{'}05)}, pages 99--106, Ann Arbor,
  Michigan. Association for Computational Linguistics.

\bibitem[{Pradhan et~al.(2012)Pradhan, Moschitti, Xue, Uryupina, and
  Zhang}]{pradhan-etal-2012-conll}
Sameer Pradhan, Alessandro Moschitti, Nianwen Xue, Olga Uryupina, and Yuchen
  Zhang. 2012.
\newblock \href {https://aclanthology.org/W12-4501} {{C}o{NLL}-2012 shared
  task: Modeling multilingual unrestricted coreference in {O}nto{N}otes}.
\newblock In \emph{Joint Conference on {EMNLP} and {C}o{NLL} - Shared Task},
  pages 1--40, Jeju Island, Korea. Association for Computational Linguistics.

\bibitem[{Shawe-Taylor and Cristianini(2004)}]{shawe-taylor_cristianini_2004}
John Shawe-Taylor and Nello Cristianini. 2004.
\newblock \href {https://doi.org/10.1017/CBO9780511809682} {\emph{Kernel
  Methods for Pattern Analysis}}.
\newblock Cambridge University Press.

\bibitem[{Shen et~al.(2018{\natexlab{a}})Shen, Lin, Huang, and
  Courville}]{shen2018neural}
Yikang Shen, Zhouhan Lin, Chin-Wei Huang, and Aaron Courville.
  2018{\natexlab{a}}.
\newblock \href {https://openreview.net/forum?id=rkgOLb-0W} {Neural language
  modeling by jointly learning syntax and lexicon}.
\newblock In \emph{International Conference on Learning Representations}.

\bibitem[{Shen et~al.(2018{\natexlab{b}})Shen, Lin, Jacob, Sordoni, Courville,
  and Bengio}]{shen-etal-2018-straight}
Yikang Shen, Zhouhan Lin, Athul~Paul Jacob, Alessandro Sordoni, Aaron
  Courville, and Yoshua Bengio. 2018{\natexlab{b}}.
\newblock \href {https://doi.org/10.18653/v1/P18-1108} {Straight to the tree:
  Constituency parsing with neural syntactic distance}.
\newblock In \emph{Proceedings of the 56th Annual Meeting of the Association
  for Computational Linguistics (Volume 1: Long Papers)}, pages 1171--1180,
  Melbourne, Australia. Association for Computational Linguistics.

\bibitem[{Smith(2011)}]{smith-linguistic-structure}
N.A. Smith. 2011.
\newblock \href {https://books.google.ch/books?id=LS3Tx59XSzcC}
  {\emph{Linguistic Structure Prediction}}.
\newblock Synthesis digital library of engineering and computer science. Morgan
  \& Claypool.

\bibitem[{Stede(2012)}]{discourse-processing}
M.~Stede. 2012.
\newblock \href {https://books.google.com/books?id=6\_xiUFpevAkC}
  {\emph{Discourse Processing}}.
\newblock Synthesis lectures on human language technologies. Morgan \&
  Claypool.

\bibitem[{Steedman(1987)}]{steedman-ccg}
Mark Steedman. 1987.
\newblock \href {http://www.jstor.org/stable/4047583} {Combinatory grammars and
  parasitic gaps}.
\newblock \emph{Natural Language \& Linguistic Theory}, 5(3):403--439.

\bibitem[{Steedman(2000)}]{syntactic-process}
Mark Steedman. 2000.
\newblock \href {https://mitpress.mit.edu/9780262692687/the-syntactic-process/}
  {\emph{The Syntactic Process}}.
\newblock MIT Press, Cambridge, MA, USA.

\bibitem[{Strzyz et~al.(2019)Strzyz, Vilares, and
  G{\'o}mez-Rodr{\'\i}guez}]{strzyz-etal-2019-viable}
Michalina Strzyz, David Vilares, and Carlos G{\'o}mez-Rodr{\'\i}guez. 2019.
\newblock \href {https://doi.org/10.18653/v1/N19-1077} {Viable dependency
  parsing as sequence labeling}.
\newblock In \emph{Proceedings of the 2019 Conference of the North {A}merican
  Chapter of the Association for Computational Linguistics: Human Language
  Technologies, Volume 1 (Long and Short Papers)}, pages 717--723, Minneapolis,
  Minnesota. Association for Computational Linguistics.

\bibitem[{Strzyz et~al.(2020)Strzyz, Vilares, and
  G{\'o}mez-Rodr{\'\i}guez}]{strzyz-etal-2020-bracketing}
Michalina Strzyz, David Vilares, and Carlos G{\'o}mez-Rodr{\'\i}guez. 2020.
\newblock \href {https://doi.org/10.18653/v1/2020.coling-main.223} {Bracketing
  encodings for 2-planar dependency parsing}.
\newblock In \emph{Proceedings of the 28th International Conference on
  Computational Linguistics}, pages 2472--2484, Barcelona, Spain (Online).
  International Committee on Computational Linguistics.

\bibitem[{Szpilrajn(1930)}]{szpilrajn}
Edward Szpilrajn. 1930.
\newblock \href {http://eudml.org/doc/212499} {Sur l'extension de l'ordre
  partiel}.
\newblock \emph{Fundamenta Mathematicae}, 16(1):386--389.

\bibitem[{Tarjan(1977)}]{tarjan}
{Robert Endre} Tarjan. 1977.
\newblock \href {https://doi.org/10.1002/net.3230070103} {Finding optimum
  branchings}.
\newblock \emph{Networks}, 7(1):25--35.

\bibitem[{Taskar et~al.(2004)Taskar, Klein, Collins, Koller, and
  Manning}]{taskar-etal-2004-max}
Ben Taskar, Dan Klein, Mike Collins, Daphne Koller, and Christopher Manning.
  2004.
\newblock \href {https://aclanthology.org/W04-3201} {Max-margin parsing}.
\newblock In \emph{Proceedings of the 2004 Conference on Empirical Methods in
  Natural Language Processing}, pages 1--8, Barcelona, Spain. Association for
  Computational Linguistics.

\bibitem[{Tay et~al.(2021)Tay, Dehghani, Abnar, Shen, Bahri, Pham, Rao, Yang,
  Ruder, and Metzler}]{long-range-arena}
Yi~Tay, Mostafa Dehghani, Samira Abnar, Yikang Shen, Dara Bahri, Philip Pham,
  Jinfeng Rao, Liu Yang, Sebastian Ruder, and Donald Metzler. 2021.
\newblock \href {https://openreview.net/forum?id=qVyeW-grC2k} {Long range arena
  : A benchmark for efficient transformers}.
\newblock In \emph{International Conference on Learning Representations}.

\bibitem[{Tesni{\`e}re(1959)}]{tesnière}
L.~Tesni{\`e}re. 1959.
\newblock \href {https://books.google.ch/books?id=6SRZAAAAMAAJ}
  {\emph{{\'E}lements de Syntaxe Structurale}}.
\newblock C. Klincksieck.

\bibitem[{Valdes et~al.(1979)Valdes, Tarjan, and Lawler}]{valdes-tarjan-lawler}
Jacobo Valdes, Robert~E. Tarjan, and Eugene~L. Lawler. 1979.
\newblock \href {https://doi.org/10.1145/800135.804393} {The recognition of
  series parallel digraphs}.
\newblock In \emph{Proceedings of the Eleventh Annual ACM Symposium on Theory
  of Computing}, STOC '79, page 1–12, New York, NY, USA. Association for
  Computing Machinery.

\bibitem[{Vendrov et~al.(2015)Vendrov, Kiros, Fidler, and
  Urtasun}]{vendrov2015order}
Ivan Vendrov, Ryan Kiros, Sanja Fidler, and Raquel Urtasun. 2015.
\newblock \href {https://arxiv.org/pdf/1511.06361.pdf} {Order-embeddings of
  images and language}.
\newblock In \emph{International Conference on Learning Representations}.

\bibitem[{Wang and Tu(2020)}]{wang-tu-2020-second}
Xinyu Wang and Kewei Tu. 2020.
\newblock \href {https://aclanthology.org/2020.aacl-main.12} {Second-order
  neural dependency parsing with message passing and end-to-end training}.
\newblock In \emph{Proceedings of the 1st Conference of the Asia-Pacific
  Chapter of the Association for Computational Linguistics and the 10th
  International Joint Conference on Natural Language Processing}, pages 93--99,
  Suzhou, China. Association for Computational Linguistics.

\bibitem[{West(2018)}]{graph-theory}
Douglas~B. West. 2018.
\newblock \href {https://books.google.com/books?id=61gtAAAACAAJ}
  {\emph{Introduction to Graph Theory}}.
\newblock Pearson Modern Classics for Advanced Mathematics Series. Pearson.

\bibitem[{Xu and Choi(2020)}]{xu-choi-2020-revealing}
Liyan Xu and Jinho~D. Choi. 2020.
\newblock \href {https://doi.org/10.18653/v1/2020.emnlp-main.686} {Revealing
  the myth of higher-order inference in coreference resolution}.
\newblock In \emph{Proceedings of the 2020 Conference on Empirical Methods in
  Natural Language Processing (EMNLP)}, pages 8527--8533, Online. Association
  for Computational Linguistics.

\bibitem[{Xue et~al.(2005)Xue, Xia, Chiou, and Palmer}]{ctb}
Naiwen Xue, Fei Xia, Fu-dong Chiou, and Marta Palmer. 2005.
\newblock \href {https://doi.org/10.1017/S135132490400364X} {The penn chinese
  treebank: Phrase structure annotation of a large corpus}.
\newblock \emph{Natural Language Engineering}, 11(2):207–238.

\bibitem[{Yamada and Matsumoto(2003)}]{yamada-matsumoto-2003-statistical}
Hiroyasu Yamada and Yuji Matsumoto. 2003.
\newblock \href {https://aclanthology.org/W03-3023} {Statistical dependency
  analysis with support vector machines}.
\newblock In \emph{Proceedings of the Eighth International Conference on
  Parsing Technologies}, pages 195--206, Nancy, France.

\bibitem[{Yang and Tu(2022)}]{yang-tu-2022-headed}
Songlin Yang and Kewei Tu. 2022.
\newblock \href {https://doi.org/10.18653/v1/2022.acl-long.155}
  {Headed-span-based projective dependency parsing}.
\newblock In \emph{Proceedings of the 60th Annual Meeting of the Association
  for Computational Linguistics (Volume 1: Long Papers)}, pages 2188--2200,
  Dublin, Ireland. Association for Computational Linguistics.

\bibitem[{Yang et~al.(2019)Yang, Dai, Yang, Carbonell, Salakhutdinov, and
  Le}]{xlnet}
Zhilin Yang, Zihang Dai, Yiming Yang, Jaime Carbonell, Russ~R Salakhutdinov,
  and Quoc~V Le. 2019.
\newblock \href
  {https://proceedings.neurips.cc/paper/2019/file/dc6a7e655d7e5840e66733e9ee67cc69-Paper.pdf}
  {{XLNet}: Generalized autoregressive pretraining for language understanding}.
\newblock In \emph{Advances in Neural Information Processing Systems},
  volume~32. Curran Associates, Inc.

\bibitem[{Yannakakis(1982)}]{yannakakis}
Mihalis Yannakakis. 1982.
\newblock \href {https://doi.org/10.1137/0603036} {The complexity of the
  partial order dimension problem}.
\newblock \emph{SIAM Journal on Algebraic Discrete Methods}, 3(3):351--358.

\bibitem[{Yli-Jyr{\"a}(2003)}]{yli2003multiplanarity}
Anssi~Mikael Yli-Jyr{\"a}. 2003.
\newblock \href
  {https://helda.helsinki.fi/server/api/core/bitstreams/ea067a64-7dde-4488-8c45-4c8a52344b2e/content}
  {Multiplanarity-a model for dependency structures in treebanks}.
\newblock In \emph{TLT 2003, Proceedings of the Second Workshop on Treebanks
  and Linguistic Theories}. V{\"a}xj{\"o} University Press.

\bibitem[{Zhang et~al.(2020)Zhang, Li, and Zhang}]{zhang-etal-2020-efficient}
Yu~Zhang, Zhenghua Li, and Min Zhang. 2020.
\newblock \href {https://doi.org/10.18653/v1/2020.acl-main.302} {Efficient
  second-order {T}ree{CRF} for neural dependency parsing}.
\newblock In \emph{Proceedings of the 58th Annual Meeting of the Association
  for Computational Linguistics}, pages 3295--3305, Online. Association for
  Computational Linguistics.

\bibitem[{Zhang and Clark(2008)}]{zhang-clark-2008-tale}
Yue Zhang and Stephen Clark. 2008.
\newblock \href {https://aclanthology.org/D08-1059} {A tale of two parsers:
  {I}nvestigating and combining graph-based and transition-based dependency
  parsing}.
\newblock In \emph{Proceedings of the 2008 Conference on Empirical Methods in
  Natural Language Processing}, pages 562--571, Honolulu, Hawaii. Association
  for Computational Linguistics.

\bibitem[{Zhou and Zhao(2019)}]{zhou-zhao-2019-head}
Junru Zhou and Hai Zhao. 2019.
\newblock \href {https://doi.org/10.18653/v1/P19-1230} {{H}ead-{D}riven
  {P}hrase {S}tructure {G}rammar parsing on {P}enn {T}reebank}.
\newblock In \emph{Proceedings of the 57th Annual Meeting of the Association
  for Computational Linguistics}, pages 2396--2408, Florence, Italy.
  Association for Computational Linguistics.

\end{thebibliography}
\bibliographystyle{acl_natbib}

\clearpage

\onecolumn

\appendix

\section{Related Work} \label{app:related-work}

\subsection{Ordinal Regression}
Ordinal regression is a family of problems that involve ranking a set of objects. 
Unlike classification, the label spaces in ordinal regression exhibit some natural ordering in its elements \cite{mccullagh-ordinal-regression}.
For instance, in information retrieval, a ranking model sorts a set of documents typically according to the document's relevance to the query.
Practically, ordinal regression can either be tackled as either regression or classification by treating the ranks as real-values
or the assignment to a particular rank value as a classification \cite{shawe-taylor_cristianini_2004}.

\subsection{Order Embeddings of Lexicons}
The notion of partial order has also been explored for learning word embeddings. 
The lexicons of natural languages exhibit hierarchical structures according to the concepts that the words represent \cite{miller-1994-wordnet}. 
For instance, `cat' and `dog' are `animal', `animal' and `plant' are `living thing'.
Order embeddings \cite{vendrov2015order,athiwaratkun2018on} propose to learn such property by learning embeddings that encode such partial order on the lexicon, resulting in improved performance on downstream tasks such as image caption retrieval.

\section{An Order-Theoretic Re-evaluation of \Cref{sec:illuminating-binary}} \label{app:revisit-binary-traversal}

    \begin{figure}[ht]
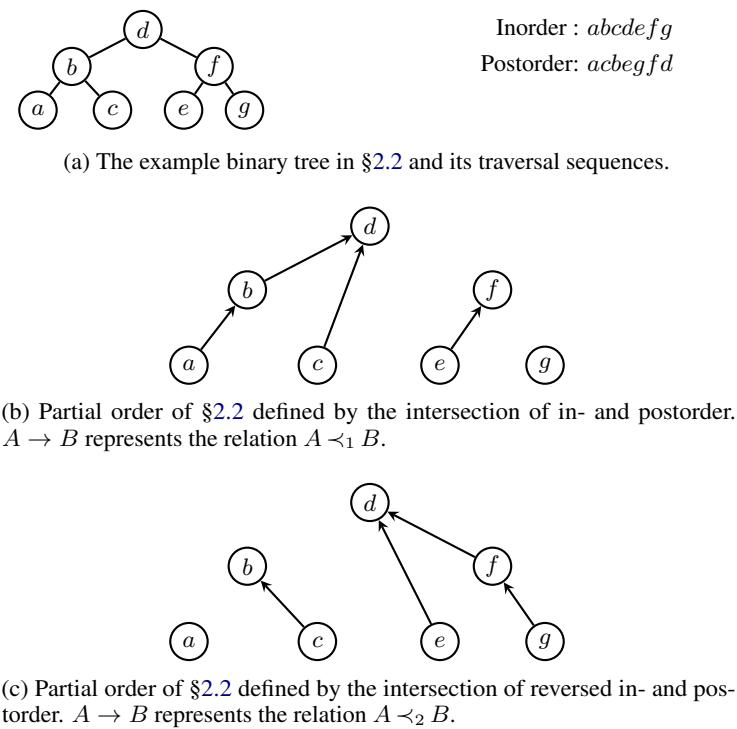

        \centering
        \begin{subfigure}{0.6\columnwidth}
            \setlength{\belowcaptionskip}{10pt}
            \begin{tabular}{p{0.59\columnwidth}r}
            \multirow{3}{0.58\columnwidth}{\extreetrav} 
            & {\small Inorder : $abcdefg$} \\
            & {\small Postorder: $acbegfd$} \\
            \end{tabular} \vspace{20pt}
            \caption{The example binary tree in \cref{sec:illuminating-binary} and its traversal sequences.}
            \label{fig:tree-trav-app-a}
        \end{subfigure}

        \begin{subfigure}{0.6\columnwidth}
            \setlength{\belowcaptionskip}{10pt}
            \centering
            \extreetravpoone
            \caption{Partial order of \cref{sec:illuminating-binary} defined by the intersection of in- and postorder. $A \to B$ represents the relation $A \rel_1 B$.}
            \label{fig:tree-trav-app-b}
        \end{subfigure}
        
        \begin{subfigure}{0.6\columnwidth}
            \setlength{\belowcaptionskip}{10pt}
            \centering
            \extreetravpotwo
            \caption{Partial order of \cref{sec:illuminating-binary} defined by the intersection of reversed in- and postorder. $A \to B$ represents the relation $A \rel_2 B$.}
            \label{fig:tree-trav-app-c}
        \end{subfigure}
        \caption{An order-theoretic re-evaluation of \cref{thm:binary-traversal}.}
        \label{fig:revisit-bin-trav}
    \end{figure}

\begin{restatable}[{\citealp[A binary tree and its traversal;][\S2.3.1, Ex.~7]{knuth-comp-programming}}]{theorem}{thmbintrav}\label{thm:binary-traversal}
    Given the inorder and either the pre- or postorder traversal of the vertices in a binary tree, the binary tree structure can be reconstructed.
\end{restatable}

\begin{proof}[Proof Sketch (order-theoretic)]
    Without loss of generality, we explain the case of the combination of in- and postorder. 
    $\vertices$ denotes the set of vertices in the binary tree.
    First, the intersection of in- and postorder defines a partial order relation $\poset_1 = \left( \vertices, \edges_1, \rel_1 \right)$.
    For any 2 vertices $\wx, \wy$ in the binary tree,
    $\wx \rel_1 \wy $ if and only if $\wx$ is a left descendant of $\wy$. I.e., $\wx$ is either the left child or a descendant of the left child of $\wy$ (see \cref{fig:tree-trav-app-b}).
    Since $\wx$ is visited before visiting $\wy$ in both inorder traversal and postorder traversal, if and only if $\wx$ is the \emph{left} descendant of $\wy$.
    The left child of each vertex in $\vertices$ can be decoded from $\poset_1$ by finding the child with the deepest subtree.
    Second, the intersection of \emph{reversed} inorder and postorder defines a partial order relation $\poset_2 = \left( \vertices, \edges_2, \rel_2 \right)$.
    For any 2 vertices $\wx, \wy$ in the binary tree,
    $\wx \rel_2 \wy $ if and only if $\wx$ is a right descendant of $\wy$ (see \cref{fig:tree-trav-app-c}).
    Since $\wx$ is visited before visiting $\wy$ in both the reversed inorder traversal and postorder traversal, if and only if $\wx$ is the \emph{right} descendant of $\wy$.
    The right child of each vertex in $\vertices$ can be decoded from $\poset_2$ also by finding the child with the deepest subtree.
    Thus, the original binary tree can be reconstructed.

\end{proof}

\section{Proofs on the Partially Ordered Properties of Structures}
\subsection{Proof of \cref{thm:vs-is-po}}
\label{sec:vs-is-po-proof}
\thmvsispo*
\begin{proof} 
    We show that a token-split structure $\poset = \left( \verticeshat, \edgeshat, \rel \right)$ satisfies all the properties of partially ordered structure defined in \cref{def:order}.

    \begin{enumerate}[noitemsep]
        \item irreflexivity: By \cref{def:token-split-4}, for all $\wx \in \verticeshat$, $\wx \norel \wx$.
        \item asymmetry: Suppose that $\exists \wx, \wy, \wx \neq \wy$, s.t. $\wx \rel \wy \land \wy \rel \wx$. By \cref{def:token-split-2,def:token-split-4}, $\wx, \wy \in \verticesr \cap \verticesb = \varnothing$. Thus, $\wx \rel \wy \implies \wy \norel \wx$.
        \item transitivity:  $\wx \rel \wy \land \wy \rel \wz$ cannot hold by \cref{def:token-split-4}. Since $\wx \rel \wy$ implies $\wx \in \verticesr \land \wy \in \verticesb$, while $\wy \rel \wz$ implies $\wy \in \verticesr \land \wx \in \verticesb$, a contradiction occurs due to $\wy \in \verticesr \cap \verticesb = \varnothing$ by \cref{def:token-split-2}. $\wx \rel \wy \land \wy \rel \wz \implies \wx \rel \wz$ holds since the antecedent of the proposition is always false.
    \end{enumerate}
    Thus, token-split structures are partially ordered.        
\end{proof}

\section{Guarantees of Order Dimension of Linguistic Structures} \label{app:efficiency-guarantees}

We justify the guarantees of order dimension of linguistic structures. One conventional way to characterize the dimension of partial orders is from a lattice-theoretic point of view.
A basic result tells us that a partial order is $2$-dimensional \emph{if and only if} its \defn{complete lattice embedding} has a \defn{planar Hasse diagram} \cite{baker-1972}. 
In other words, its complete lattice embedding can be drawn on a plane without any crossing edges.\looseness=-1

\begin{theorem}[{\citealp[Thm. 4.1]{baker-1972}}] \label{thm:planar-dim-2}
Suppose $\poset = \VERel$ is a partially ordered structure. Then the following are equivalent:
    \begin{enumerate}[noitemsep,label=\textnormal{(\alph*)},wide]
        \item $\orderdimfunc(\poset) \leq 2.$
        \item The complete lattice embedding of $\poset$ has a planar Hasse diagram.
    \end{enumerate}
\end{theorem}

\begin{remark}
\citet{macneille-1937} and \citet[Chapter 5]{birkhoff-1967-lattice} introduced the construction of complete lattice embeddings for any partial order.
Although it is difficult in practice to compute the complete lattice embedding for a partially ordered structure \cite{macneille-1937}, \Cref{thm:planar-dim-2} can still provide an empirical characterization of the class of structures that can be efficiently represented.
According to Euler's formula, the average degree of a vertex in a planar graph cannot exceed 6 \cite[\S 6.1.23]{graph-theory}, which intuitively forces the partially ordered structures that can be represented as an intersection of 2 totally ordered structures to be sparse enough---thus to have planar complete lattice embeddings.

Fortunately, this is often the case in natural language.
Such phenomenon is closely related to what is termed \defn{valency} by \citet[Part 1, Book D]{tesnière}.
The number of actants (i.e., arguments) needed to implement the function of a word is a property of the word itself---a constant that does not change with the context (cf. \defn{categories}\footnote{E.g., the English word ``\textit{give}'' may have the category (VP/NP)/NP, meaning that it needs two NP categories to the right to form a VP. An example is the verb phrase ``\textit{give me an apple}'', in which ``\textit{me}'' and ``\textit{an apple}'' are noun phrases.} in categorial grammars \cite{Ajdukiewicz-35,bar-hillel-ccg,steedman-ccg}). 
In natural language, the valency of a word is often a small constant. 
For instance, \citet[Chapter 3, fn. 10 and Chapter 8, p. 212]{syntactic-process} observes that the highest valency in the Dutch and English lexicon can be regarded as bounded by $4$.
\end{remark}

We refer interested readers to \citet{macneille-1937} and \citet[Chapter 5]{birkhoff-1967-lattice} for the construction of complete lattice embeddings.
Here, we give a weaker but more practical efficiency guarantee, based on a method to construct large partially ordered structures from smaller partially ordered structures.\looseness=-1
\begin{definition}[{\citealp[Series-Parallel Partial Orders;][]{valdes-tarjan-lawler}}] \label{def:series-parallel-posets}
A partially ordered structure is \defn{series-parallel} if it satisfies the following inductive definition:\looseness=-1
    \begin{deflist}[noitemsep,nosep,label=\textnormal{(\alph*)},ref=\textnormal{\thetheorem\,(\alph*)}]
        \item A single-vertex structure with no edges is series-parallel; 
        \item If partially ordered structures $\poset_{1} = \left( \vertices_{1}, \edges_{1}, \rel \right) $ and $\poset_{2} = \left( \vertices_{2}, \edges_{2}, \rel \right)$ are series-parallel,
        so is the partially ordered structures constructed by \emph{either} of the following operations: \label{def:s-p-posets-multiple}
        \begin{deflist}[noitemsep,nosep,label=\textnormal{\roman*.},ref=\textnormal{\thetheorem\,\alph{deflisti}.\roman*},leftmargin=8pt,labelindent=20pt]
            \item Parallel composition: \\  $\posetp = \left( \vertices_{1} \cup \vertices_{2}, \edges_{1} \cup \edges_{2}, \rel \right)$. \label[definition]{def:s-p-posets-parallel}
            \item Series composition: \\  $ \posets = \left( \vertices_{1} \cup \vertices_{2}, \edges_{1} \cup \edges_{2} \cup (\sink_1\!\times\!\source_2), \rel \right) $, 
            where $\sink_1$ is the set of sinks of $\poset_{1}$ and $\source_2$ the set of sources of $\poset_{2}$.\footnote{Sources and sinks refer to the vertices without incoming arcs and without outgoing arcs, respectively.}
        \end{deflist}
    \end{deflist}
\end{definition}

\begin{theorem}[{\citealp[Series-parallel partially ordered structures are 2-dimensional;][]{valdes-tarjan-lawler}}] \label{thm:series-parallel-2d}
    The dimension of series-parallel partially ordered structures is at most 2.\looseness=-1
\end{theorem}

\Cref{thm:series-parallel-2d} provides the guarantee that many structures in natural language processing can be represented as the intersection of 2 totally ordered structures.
Since most structures of interest in NLP, such as trees and forests (thereby alignments and set partitioning), can be subsumed under series-parallel partially ordered structures, therefore have an order dimension of at most 2.

\begin{proposition}[{\citealp[Trees are 2-dimensional;][]{lawler-trees-series-parallel}}] \label{prop:tree-2d}
    Directed tree partially ordered structures are series-parallel. 
    The order dimension of tree structures is at most 2.
\end{proposition}

\begin{proposition}[Forests are 2-dimensional] \label{prop:forest-2d}
    Forests are series-parallel. 
    The order dimension of forest structures is at most 2.
\end{proposition}
\begin{proof}
    Forests are parallel compositions of trees.
    Thus, the proposition holds. 
\end{proof}

\section{Efficient Algorithm for $\oplus$-Aggregation} \label{app:decoding-general}

\subsection{Correctness of \cref{alg:efficient-computation}}
\efficientcomputation*
\begin{proposition}
    In \cref{alg:efficient-computation}, $G_1 = \bigoplus_{x \in \vertices} \bigoplus_{y \in\calS_{1}(\wx)} \left(\rankfunc^{(1)}(\wx) -\rankfunc^{(1)}(\wy) \right) $.
\end{proposition}
\begin{proof}
    By induction, we show that upon finishing step $n$, $s_1 = \bigoplus_{\wy \in\calS_{1}( \sortedvertices_{n+1})} -\rankfunc^{(1)}(\wy)  $, $G_1 = \bigoplus_{\wx \in \{ \sortedvertices_{1}, \cdots, \sortedvertices_{n} \} } \bigoplus_{\wy \in\calS_{1}(\wx)} \left(\rankfunc^{(1)}(\wx) -\rankfunc^{(1)}(\wy) \right) $. 
    First, $\calS_{1}( \sortedvertices_{n}) = \{ \sortedvertices_{1}, \cdots, \sortedvertices_{n-1} \} $ holds as discussed in \cref{sec:efficient-algorithm}.
    When $n = 1$, we have $s_1 = -\rankfunc^{(1)}(\sortedvertices_1) $, $G_1 = \bzero =  \bigoplus_{\wx \in \{ \sortedvertices_{1} \} } \bigoplus_{y \in\calS_{1}(\wx)} \left(\rankfunc^{(1)}(\wx) -\rankfunc^{(1)}(\wy) \right) $, since $\calS_{1}(\sortedvertices_{1}) = \varnothing$.
    Assume that our statements hold for $n = j$, when $n = j + 1$, it is straightforward that $s_1 = \bigoplus_{\wy \in\calS_{1}( \sortedvertices_{j+2})} -\rankfunc^{(1)}(\wy)$.
    For $G_1$, we have
    \begin{subequations}
        \begin{align}
            G_1 &= \bigoplus_{\wx \in \{ \sortedvertices_{1}, \cdots, \sortedvertices_{j} \} } \bigoplus_{\wy \in\calS_{1}(\wx)} \mleft( \rankfunc^{(1)}(\wx) -\rankfunc^{(1)}(\wy) \mright) \oplus \mleft( \rankfunc^{(1)}(\sortedvertices_{j+1}) + \bigoplus_{\wy \in\calS_{1}( \sortedvertices_{j+1})} -\rankfunc^{(1)}(\wy) \mright) \\
             &= \bigoplus_{\wx \in \{ \sortedvertices_{1}, \cdots, \sortedvertices_{j} \} } \bigoplus_{\wy \in\calS_{1}(\wx)} \mleft(\rankfunc^{(1)}(\wx) -\rankfunc^{(1)}(\wy) \mright) \oplus \bigoplus_{\wy \in\calS_{1}( \sortedvertices_{j+1})} \mleft( \rankfunc^{(1)}(\sortedvertices_{j+1}) -\rankfunc^{(1)}(\wy) \mright) \\
             &= \bigoplus_{\wx \in \{ \sortedvertices_{1}, \cdots, \sortedvertices_{j+1} \} } \bigoplus_{\wy \in\calS_{1}(\wx)} \mleft(\rankfunc^{(1)}(\wx) -\rankfunc^{(1)}(\wy) \mright)
        \end{align}
    \end{subequations}
    Thus, the claims hold for $n = j+1$, establishing the induction step.
\end{proof}

\begin{proposition}
    \cref{alg:efficient-computation} runs in $\bigo{\seqlen}$ time and space.
    With parallel computing, \cref{alg:efficient-computation} runs in $\bigo{\log \seqlen}$ span.
\end{proposition}

\begin{proof}
    The sorting step in \cref{line:alg-1-sort} can be executed in $\bigo{\seqlen}$ time and space.
    The for loop in \crefrange{line:alg-1-for-1}{line:alg-1-for-4} runs in $\bigo{\seqlen}$ time and space. 
    In total, \cref{alg:efficient-computation} runs in $\bigo{\seqlen}$ time and space. 
    Computing $s_1$ in each step is a prefix-sum of $-\rankfunc^{(1)}(U_n)$, which can be done in $\bigo{\log \seqlen}$ span with parallel computing.
    $q_1, G_1$ in each step can be computed in $\bigo{1}$ in parallel following the computation of all $s_1$. 
    Thus, the total span of \cref{alg:efficient-computation} is $\bigo{\log \seqlen}$.
\end{proof}

\subsection{Order Dimension $\orderdim > 2$} \label{app:dim-gt-2} %

Finding all $\wy \in \vertices$ such that $\wx \rel \wy$ in a partial order for a given $\wx \in \vertices$ requires efficiently finding all $\wy$ that satisfy $\bigwedge_{k \in \Nlessthan{\orderdim}} (\rankfunck(\wx) < \rankfunck(\wy))$.
We remark that this problem bears a resemblance to \defn{orthogonal range searching} in a $\orderdim$-dimensional space \cite[Chapter 5]{comp-geometry}, i.e., for a given $\wx$, we aim to find all $\wy$ such that $(\rankfunc^{(1)}(\wy), \rankfunc^{(2)}(\wy), \cdots, \rankfunc^{(\orderdim)}(\wy))$ is within the range $ (\rankfunc^{(1)}(\wx), \infty) \times (\rankfunc^{(2)}(\wx), \infty) \times \cdots \times (\rankfunc^{\orderdim}(\wx), \infty) $.
This problem can be naïvely solved in $\bigo{\log^{\orderdim-1} \seqlen + \ell}$ using a \defn{range tree} \cite{range-tree-bentley,multidim-bentley,semigroup-chazelle,range-tree-1-chazelle,range-tree-2-chazelle}, where $\ell$ is the cardinality of query results, as opposed to arc-factored models in which solving the same problem takes $\bigo{\seqlen}$ computations.

For $\oplus$-aggregation, a more efficient algorithm which makes use of $(\orderdim\!-\!1)$-dimensional range trees can be designed. 
In future work, we show that computing the complexity of $\oplus$-aggregation for \emph{all} $\wx \in \vertices$ can be further reduced to $\bigo{\orderdim \seqlen \log^{\orderdim-2} \seqlen}$ by applying \citeposs{fredman-trick} trick which we used in \cref{alg:efficient-computation}.
Extending the notation in \cref{sec:efficient-algorithm}, the set of all vertices $\vertices$ can be partitioned into $\orderdim$ subsets $\calS_{1}(\wx), \cdots, \calS_{\orderdim}(\wx)$ for each $\wx \in \vertices$, where $\calS_{k}(\wx) = \{ \wy \mid \wy \in \vertices \land \pairFunc(\wx, \wy) = \rankfunc^{(k)}(\wx) -\rankfunc^{(k)}(\wy) \} $. $\semiringplus_{\wy \in \vertices} \pairFunc(\wx, \wy)$ can be decomposed into a $\oplus$-aggregation of $\orderdim$ terms.
\begin{subequations}
\begin{align}
    G(\wx)  &\defeq  \semiringplus_{\wy \in \vertices} \pairFunc(\wx, \wy) \\ %
    G(\wx)  &= \semiringplus_{k \in \Nlessthan{\orderdim}} \underbrace{\mleft( \semiringplus_{\wy \in \calS_{k}} \pairFunc(\wx, \wy) \mright)}_{\defeq G_k(\wx)}
\end{align}
\end{subequations}
We leave to future work showing that computing each $G_k(\wx)$ takes $\bigo{ \log^{\orderdim-2} \seqlen}$.

\section{Hyperparameter Settings} \label{app:hyperparam}
\subsection{Dependency Parsing} \label{app:hyperparam-dep}
For pretrained language models, we use \texttt{XLNet-large-cased}\footnote{\url{https://huggingface.co/xlnet-large-cased}} \cite{xlnet} for PTB, \texttt{bert-base-chinese}\footnote{\url{https://huggingface.co/bert-base-chinese}} for CTB, and \texttt{bert-base-multilingual-cased}\footnote{\url{https://huggingface.co/bert-base-multilingual-cased}} for UD.
We set the dimension of POS tag embedding to 256 for all experiments. 
On top of concatenated pretrained representations and POS embedding, we use a 3-layer BiLSTM \cite{lstm} with a hidden size of 768 for base-sized models (\texttt{bert-base-chinese} on CTB and \texttt{bert-multilingual-cased} on UD) and 1024 for large-sized models (\texttt{xlnet-large-cased} on PTB).
We apply dropout with a rate of 0.33 to the concatenated embedding layer, between LSTM layers, and before the linear projection layer of the realizer. 
We employ AdamW \cite{loshchilov2018decoupled} with a learning rate of $2\mathrm{e}{-5}$ for pretrained LMs and $1\mathrm{e}{-4}$ for POS embedding, BiLSTM, and linear projection during training. The gradient clipping threshold is set to $1.0$. The batch size for training is $32$. The number of training epochs is $50$.

\subsection{Coreference Resolution} \label{app:hyperparam-coref}
We use \texttt{longformer-large-cased}\footnote{\url{https://huggingface.co/allenai/longformer-large-4096}} 
 \cite{Beltagy2020Longformer} as the pretrained encoder.
 We use the same hyperparameter settings as \citet{kirstain-etal-2021-coreference}.
We use AdamW with a learning rate of $1\mathrm{e}{-5}$ for pretrained LM and $3\mathrm{e}{-4}$ for the linear projection during training, with 5600 linear warmup steps.
Training documents are batched into batches with maximum 5000 tokens in total. 
The number of training epochs is $129$.\looseness=-1

\section{Datasets} \label{app:datasets}

\subsection{Dependency Parsing} \label{app:dataset-dep}

\paragraph{Preprocessing.}
We follow previous work \cite{kiperwasser-goldberg-2016-simple,dozat-biaffine} to derive the dependency annotations from the treebank annotations using the Stanford Dependency converter v3.3.0 \cite{de2008stanford}. 
During evaluation, punctuations are omitted. 
Following \citet{hexa}, we provide gold part-of-speech tags to the model during training and decoding.

\paragraph{Splits.} 
The dataset splits are consistent with previous work. 
For PTB, we follow the standard split of \citet{marcus-etal-1993-building}, resulting in 39,832 sentences for training, 1,700 for development, and 2,416 for testing. 
For CTB, we follow the split of \citet{zhang-clark-2008-tale}, resulting in 16,091 sentences for training, 803 for development, and 1,910 for testing.
For UD, we follow previous work \cite{zhang-etal-2020-efficient,yang-tu-2022-headed} and use the standard splits of the following corpora for experiments:  BG-btb, CA-ancora, CS-pdt, DE-gsd, EN-ewt, ES-ancora, FR-gsd, IT-isdt, NL-alpino, NO-rrt, RO-rrt, RU-syntagrus.

\paragraph{Licenses.} 
The PTB and CTB datasets are licensed under LDC User Agreement. 
The UD dataset is licensed under the Universal Dependencies License Agreement.

\subsection{Coreference Resolution} \label{app:dataset-coref}

\paragraph{Preprocessing.}
We experiment on the CoNLL-2012 English shared task dataset \cite[OntoNotes;][]{pradhan-etal-2012-conll}. 
We follow the preprocessing procedure of \cite{kirstain-etal-2021-coreference}. 
During training and decoding, the speaker information is provided to the model.

\paragraph{Splits.}
The OntoNotes dataset contains 2,802 documents for training, 343 for validation, and 348 for testing.
We use this official split following previous work \cite{lee-etal-2017-end,kirstain-etal-2021-coreference}.

\paragraph{Licenses.}
The OntoNotes dataset is licensed under LDC User Agreement.

\section{Efficiency Evaluation} \label{app:efficiency-eval}

\subsection{Dependency Parsing}
For efficiency evaluation, \texttt{BERT-large-cased}\footnote{\url{https://huggingface.co/bert-large-cased}} is used as the pretrained encoder for our method with $\orderdim=2$, hexatagger \cite[\Hexataggertt{};][]{hexa}, and biaffine model (\texttt{Biaff}).
We use the English PTB test set and truncate or pad the input sentences to the control length. 
The results are averaged over 3 random runs on the same server with one Nvidia A100-80GB GPU. The other experimental settings are kept the same (i.e., the version of PyTorch and \texttt{transformers}, FP32 precision, batching).

\subsection{Coreference Resolution}
\begin{table}[ht]
\centering
\resizebox{0.65\columnwidth}{!}{%
\begin{tabular}{@{}ccccc@{}}\toprule
 & \multicolumn{2}{c}{Speed (doc/s) $\uparrow$} & \multicolumn{2}{c}{Memory (GB) $\downarrow$} \\
\cmidrule(lr){2-3}
\cmidrule(lr){4-5}
Doc length & \texttt{Ours} ($\orderdim=4$)  & \citeauthor{kirstain-etal-2021-coreference} & \texttt{Ours} ($\orderdim=4$) & \citeauthor{kirstain-etal-2021-coreference} \\ \midrule
512 & 72.5 & 35.7 & 7.3 & 7.4 \\
1024 & 54.3 & 26.7 & 7.3 &  7.4 \\
2048 & 33.8 & 15.9 & 9.4 &  9.5 \\
4096 & 19.3 & 8.6 & 17.8 &  21.0  \\ \midrule
 {overall} & \textbf{82.8} & 41.9  & \textbf{7.3} &  7.4 \\ \bottomrule
\end{tabular}}
\caption{Comparison of speed and memory consumption on OntoNotes test set using Longformer-base\footnotemark{} as pretrained encoder. 
Results are averaged over 3 random runs on the same server with one Nvidia A100-80GB GPU using BERT-large as encoder. We use a batch size of 32 documents.}
\label{tab:efficiency-ontonotes}
\end{table}
\footnotetext{\url{https://huggingface.co/allenai/longformer-base-4096}}

We compare the efficiency of our order-theoretic method with baseline coreference resolution model.
The full results are given in \Cref{tab:efficiency-ontonotes}.
On the OntoNotes coreference resolution benchmark, our method is twice as fast as \citeposs{kirstain-etal-2021-coreference} model while using less memory, especially on long documents.
It is worth noting that \citet{kirstain-etal-2021-coreference} has already performed aggressive optimization in both the speed and memory footprint of coreference modeling.
I.e., they abandon the computation for textual span representations and entity-pair representations, and use biaffine scorers to compute coreference scores.

\section{Additional Experimental Results}
\subsection{Dependency Parsing} \label{app:ud-extra}
We report additional experimental results on the UD dependency parsing dataset in \cref{tab:ud}. 
On average, our model has state-of-the-art performance and outperforms all other baseline models on 5 languages.

\begin{table}[h]
\centering
\resizebox{0.99\columnwidth}{!}{%
\begin{tabular}{lccccccccccccc}
        \toprule
            & 	bg   & 	 ca   &  cs   &  de   &  en  &  es  &  fr   &  it  & nl  & no  & ro  & ru   &  Avg. \\ \toprule
        \citet{zhang-etal-2020-efficient}   &  90.77   &  91.29   &  91.54   &  80.46   &  87.32   &  90.86   &  87.96   &  91.91   &  88.62   &  91.02   &  86.90   &  93.33   &  89.33 \\
        \citet{wang-tu-2020-second}   &  90.53   &  92.83   &  92.12   &  81.73   &  89.72   &  92.07   &  88.53   &  92.78   &  90.19   &  91.88   &  85.88   &  92.67   &  90.07   \\ 	\midrule
    \multicolumn{13}{c}{+$\text{BERT}_{\text{multilingual}}$}  \\  \midrule 
        \citet{wang-tu-2020-second}   &  {91.30}   &  93.60   &  92.09   &  82.00   &  90.75   &  92.62   &  89.32   &  93.66   &  91.21   &  91.74   &  86.40   &  92.61    &  90.61 \\ 
        \citet{dozat-biaffine}   &  90.30   &  \textbf{94.49}   &  {92.65}   &  \textbf{85.98}   &  91.13   &  93.78   &  \textbf{91.77}   & 94.72   & 91.04   & 94.21   & 87.24   &  \textbf{94.53}   & 91.82 \\ 
         \citet{yang-tu-2022-headed}   &  91.10   &  94.46   & 92.57   &  85.87   &  91.32   &  \textbf{93.84}   &  91.69   &  \textbf{94.78}   &  {91.65}   &  \textbf{94.28}   &  {87.48}   & 94.45   &  {91.96} \\
         \citet{hexa}   &  92.87   &  93.79   &  92.82   &  85.18   &  90.85   &  93.17   &  91.50   &  94.72   &  91.89 &  {93.95}   &  87.54   &  94.03   &  {91.86} \\ \midrule
         \texttt{ours} ($\orderdim=2$)   &  92.81 & 93.26 & 92.52 & 83.33 & 90.38 & 92.55 & 89.83 & 93.82 & 91.29 & 93.61 & 87.40 & 94.10 & 91.24 \\
         \texttt{ours} ($\orderdim=4$)   &  \textbf{93.82} & 94.23 & \textbf{93.03} & 84.68 & \textbf{91.40} & 93.62 & 90.95 & 94.59 & \textbf{92.58} & 94.22 & \textbf{88.45} & 94.40  & \textbf{92.16} \\ \bottomrule
\end{tabular}}
\caption{LAS scores on the test sets of 12 languages in UD. Our method with an order dimension of $\orderdim=4$ achieves competitive performance in all languages, being state-of-the-art on 5 languages and on average.}
\label{tab:ud}
\end{table}

\section{Qualitative Examples}  \label{app:qual-ex}
We present some qualitative examples from the PTB development set and one non-projective example using our method with $a$ 2-dimensional realizer, with their ground truth annotations on the right in \crefrange{fig:qual-ex-20}{fig:qual-ex-0}.
For a more intuitive and compact exhibition, we plot the 2 total orders output by our model in a 2-dimensional plane.
Each axis corresponds to one of the 2 orders.
The relation $\wx \rel \wy$ encoded by $\bigwedge_{k\in \{1,2\}} \rankfunck(\wx) < \rankfunck(\wy)$ is equivalent to $\wx$ being located \emph{below and to the left} of $\wy$.

Tokens in $\verticesr$ and $\verticesb$ are represented by \redcross{} and \bluecircle, respectively.
The line segments between \redcross{} and \bluecircle{} are the extracted dependency relations.
In each of the plots, every \redcross{} (token in $\verticesr$) except for the root is connected to a \bluecircle{} (token in $\verticesb$), which indicates \redcross{} is the modifier of \bluecircle{}.
The roots (\textit{about, moving, ready, had, adds, bought} represented by \redcross) are not connected to any other word.

\begin{figure}[ht]
    \centering
    \begin{subfigure}{0.51\textwidth}
        \includegraphics{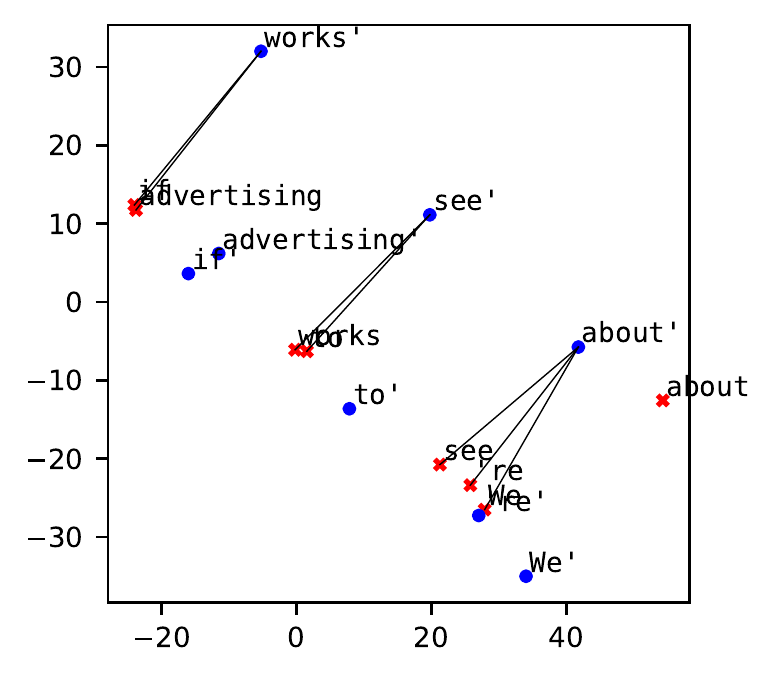}
    \end{subfigure}
    \begin{subfigure}[b]{0.45\textwidth}
        \begin{dependency}[theme=simple,edge horizontal padding=3pt]
        \begin{deptext}[column sep=0.1em,font=\small]
        We \& 're \& about \& to \& see \& if \& advertising \& works \\
        \end{deptext}
        \depedge{3}{1}{}
        \depedge{3}{2}{}
        \deproot{3}{}
        \depedge{5}{4}{}
        \depedge{3}{5}{}
        \depedge{8}{6}{}
        \depedge{8}{7}{}
        \depedge{5}{8}{}
        \end{dependency}
    \end{subfigure}
    \caption{\textit{We 're about to see if advertising works}}
    \label{fig:qual-ex-20}
\end{figure}

\begin{figure}[ht]
    \centering
    \begin{subfigure}{0.5\textwidth}
        \includegraphics{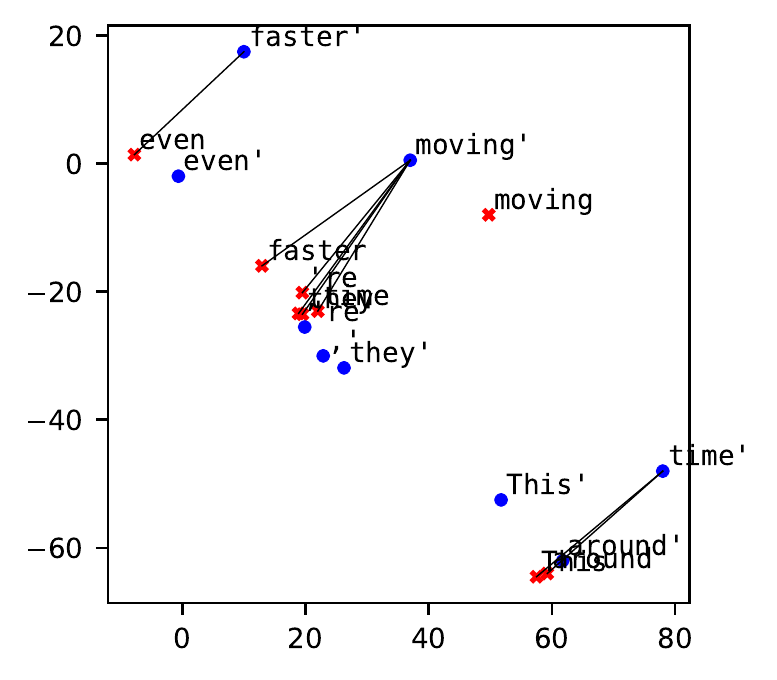}
    \end{subfigure}
    \begin{subfigure}[b]{0.45\textwidth}
        \begin{dependency}[theme=simple,edge horizontal padding=3pt]
        \begin{deptext}[column sep=0.05em,font=\small]
        This \& time \& around \& , \& they \& 're \& moving \& even \& faster \\
        \end{deptext}
        \depedge{2}{1}{}
        \depedge{7}{2}{}
        \depedge{2}{3}{}
        \depedge{7}{4}{}
        \depedge{7}{5}{}
        \depedge{7}{6}{}
        \deproot{7}{}
        \depedge{9}{8}{}
        \depedge{7}{9}{}
        \end{dependency}
    \end{subfigure}
    \caption{\textit{This time around , they 're moving even faster }}
    \label{fig:qual-ex-24}
\end{figure}

\begin{figure}[ht]
    \centering
    \begin{subfigure}{0.51\textwidth}
        \includegraphics{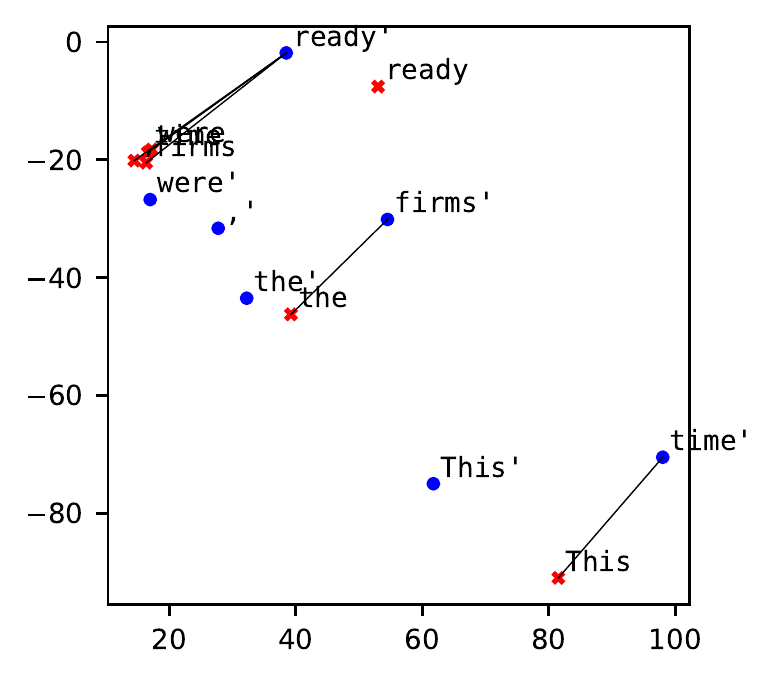}
    \end{subfigure}
    \begin{subfigure}[b]{0.45\textwidth}
        \begin{dependency}[theme=simple,edge horizontal padding=3pt]
        \begin{deptext}[column sep=0.1em,font=\small]
        This \& time  \& , \& the \& firms \& were \& ready  \\
        \end{deptext}
        \depedge{2}{1}{}
        \depedge{7}{2}{}
        \depedge{7}{3}{}
        \depedge{5}{4}{}
        \depedge{7}{5}{}
        \depedge{7}{6}{}
        \deproot{7}{}
        \end{dependency}
    \end{subfigure}
    \caption{\textit{This time , the firms were ready }}
    \label{fig:qual-ex-31}
\end{figure}

\begin{figure}[ht]
    \centering
    \begin{subfigure}{0.51\textwidth}
        \includegraphics{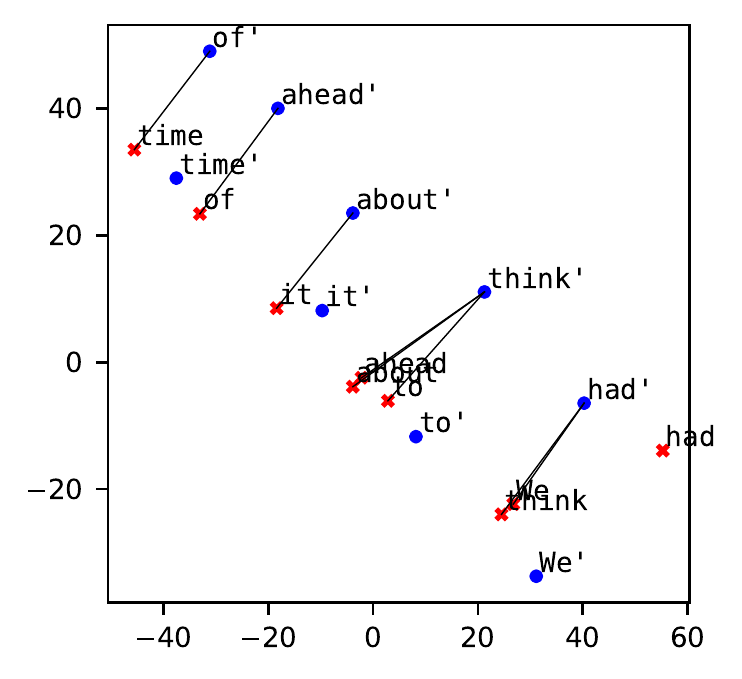}
    \end{subfigure}
    \begin{subfigure}[b]{0.45\textwidth}
        \begin{dependency}[theme=simple,edge horizontal padding=3pt]
        \begin{deptext}[column sep=0.1em,font=\small]
        We \& had  \& to \& think \& about \& it \& ahead \& of \& time  \\
        \end{deptext}
        \depedge{2}{1}{}
        \deproot{2}{}
        \depedge{4}{3}{}
        \depedge{2}{4}{}
        \depedge{4}{5}{}
        \depedge{5}{6}{}
        \depedge{4}{7}{}
        \depedge{7}{8}{}
        \depedge{8}{9}{}
        \end{dependency}
    \end{subfigure}
    \caption{\textit{We had to think about it ahead of time }}
    \label{fig:qual-ex-39}
\end{figure}

\begin{figure}[ht]
    \centering
    \begin{subfigure}{0.51\textwidth}
        \includegraphics{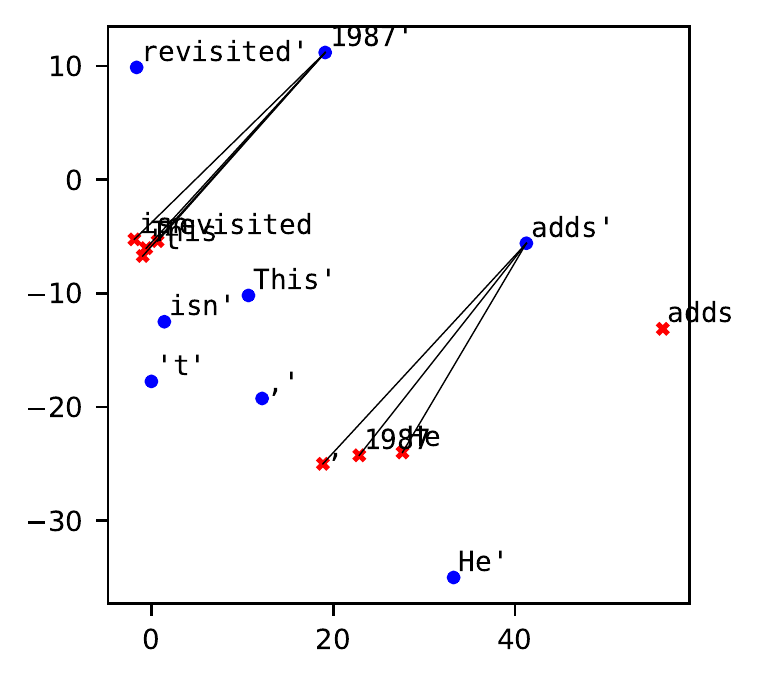}
    \end{subfigure}
    \begin{subfigure}[b]{0.45\textwidth}
        \begin{dependency}[theme=simple,edge horizontal padding=3pt]
        \begin{deptext}[column sep=0.1em,font=\small]
        He \& adds  \& , \& This \& isn \& 't \& 1987 \& revisited   \\
        \end{deptext}
        \depedge{2}{1}{}
        \deproot{2}{}
        \depedge{2}{3}{}
        \depedge{7}{4}{}
        \depedge{7}{5}{}
        \depedge{7}{6}{}
        \depedge{2}{7}{}
        \depedge{7}{8}{}
        \end{dependency}
    \end{subfigure}
    \caption{\textit{He adds , " This isn 't 1987 revisited "
}}
    \label{fig:qual-ex-54}
\end{figure}

\begin{figure}[ht]
    \centering
    \begin{subfigure}{0.51\textwidth}
        \includegraphics{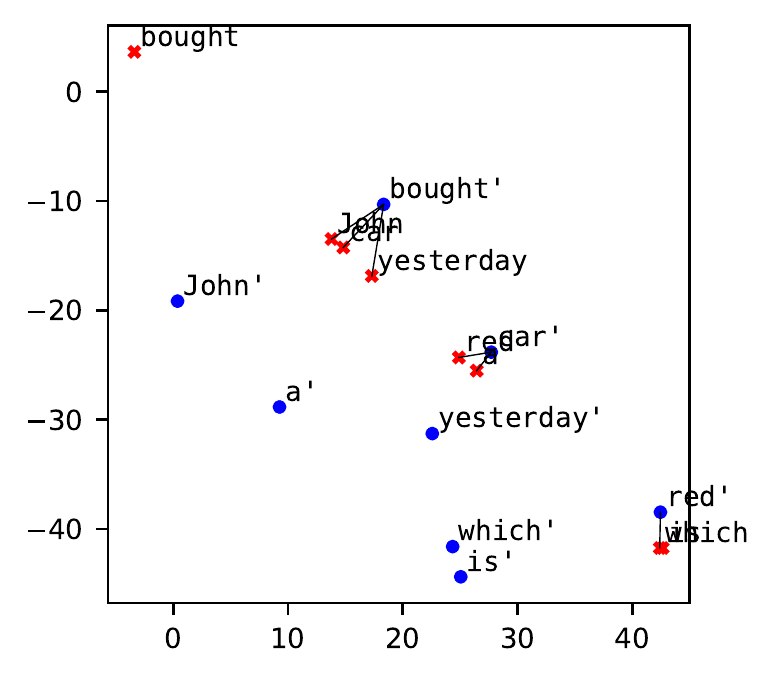}
    \end{subfigure}
    \begin{subfigure}[b]{0.45\textwidth}
        \begin{dependency}[theme=simple,edge horizontal padding=3pt]
        \begin{deptext}[column sep=0.1em,font=\small]
        John \& bought  \& a \& car \& yesterday \& which \& is \& red   \\
        \end{deptext}
        \depedge{2}{1}{}
        \deproot{2}{}
        \depedge{4}{3}{}
        \depedge{2}{4}{}
        \depedge{2}{5}{}
        \depedge{8}{6}{}
        \depedge{8}{7}{}
        \depedge{4}{8}{}
        \end{dependency}
    \end{subfigure}
    \caption{A sentence with a non-projective dependency structure: \textit{John bought a car yesterday which is red}}
    \label{fig:qual-ex-0}
\end{figure}

\end{document}